%% file: main.tex
\documentclass{article} 
\usepackage{iclr2027_conference,times}

\input{math_commands.tex}

\usepackage{algorithm}
\usepackage{algpseudocode}
\usepackage{hyperref}
\usepackage{url}
\usepackage{booktabs,array}
\usepackage{graphicx}
\usepackage{subcaption}
\usepackage{float}
\usepackage{xcolor}

\title{
    Bilinear World Models: {\fontsize{14}{10}\selectfont
    Learning Representations with Structured Dynamics for Efficient Control}
}

\author{Antonio Pariente$^*$, \;
Ignacio Boero\thanks{Equal contribution.}, \;
Nikolai Matni \&
Alejandro Ribeiro  \\
University of Pennsylvania\\
\texttt{\{pariente,iboero,nmatni,aribeiro\}@engineering.upenn.edu}}

\iclrfinalcopy 
\begin{document}

\maketitle
\lhead{Preprint.}

\begin{abstract}
World models jointly learn latent representations and dynamics that predict how high-dimensional observations evolve under actions.
In this work, we propose a JEPA-style world model in which, rather than learning arbitrary latent dynamics, we restrict them to follow a bilinear parameterization.
This structure enables efficient planning and control while shifting the modeling burden onto the encoder, encouraging richer representations that expose the controllable geometry of the system.
In particular, this structured parameterization allows us to structurally enforce action recoverability, thereby preventing representation collapse by construction.
Although prescribing a bilinear parametrization may appear restrictive, we show that a broad class of nonlinear dynamical systems admits a transformation under which the dynamics become bilinear.
Empirically, we show across standard 2D and 3D control tasks that representations with bilinear-parameterized dynamics can be learned directly from high-dimensional observations, reducing planning time by nearly three orders of magnitude while retaining or even improving control accuracy.
We also propose more demanding regimes of longer-horizon planning and real-time control, and demonstrate that our method succeeds in both, moving JEPA-style world models beyond short-horizon offline planning.
\end{abstract}

\input{01_introduction/input}
\input{02_problem_formulation/input}

\input{03_koopman_theory/input}
\input{04_control/input}

\input{05_experiments/input}

\input{statements}

\bibliography{References}
\bibliographystyle{iclr2027_conference}

\appendix
\input{06_appendix/input}

\end{document}

%% file: math_commands.tex
\usepackage{amsthm}

\newtheorem{proposition}{Proposition}

\newtheorem{assumption}{Assumption}

\newtheorem{lemma}{Lemma}

\usepackage{amsmath,amsfonts,bm}

\def\eqref#1{equation~\ref{#1}}

\def\1{\bm{1}}

\DeclareMathAlphabet{\mathsfit}{\encodingdefault}{\sfdefault}{m}{sl}
\SetMathAlphabet{\mathsfit}{bold}{\encodingdefault}{\sfdefault}{bx}{n}



%% file: 01_introduction/input.tex
\section{Introduction}
\label{sec:intro}

\input{01_introduction/a_intro}

\input{01_introduction/b_related}

%% file: 01_introduction/a_intro.tex
World models learn a latent representation of an environment together with dynamics that predict how it evolves under actions. This allows planning from high-dimensional observations, such as images, to be performed in a lower-dimensional latent space. Modern approaches typically use neural networks to learn both the encoder that maps observations to latent states and the dynamical model that predicts their evolution. We argue that such flexibility in the dynamical model is often unnecessary and counterproductive for planning and control, motivating structured dynamics.

\begin{itemize}\item[(C1)]We propose to learn world models with structured, bilinear-parameterized latent dynamics.\end{itemize}

Rather than learning arbitrary latent dynamics, we restrict their state dependence through a bilinear parameterization followed by a structured normalization. This requires the encoder to discover a representation that satisfies this structure, thus shifting expressivity from the dynamical predictor to the encoder itself. An immediate consequence is that actions can be recovered from latent transitions, thus preventing representation collapse by construction. When choosing the structure to impose, it is desired to have dynamics that enable efficient planning and control while remaining expressive enough to represent a broad class of systems under an appropriate transformation. A bilinear parameterization satisfies both requirements.

\begin{itemize}\item[(C2)] We show that bilinear latent dynamics are expressive for a broad class of nonlinear systems and that the resulting structured parameterization enables efficient control.\end{itemize}

To establish expressiveness, we leverage Koopman theory and show that any action-affine nonlinear system admits a transformation under which its dynamics become bilinear. This shows that fixing bilinear dynamics still allows a broad class of nonlinear systems to be represented. The same structure is then exploited for control by replacing computationally expensive zeroth-order methods such as CEM with a gradient-based planner that directly uses the sensitivity of the predicted endpoint to the action sequence. This enables substantially more efficient optimization in the latent space while preserving strong control performance. Both advantages are reflected directly in our experiments.

\begin{itemize}\item[(C3)] Bilinear World Models retain or improve control performance while reducing planning time by three orders of magnitude. \end{itemize}

We evaluate Bilinear World Models on four standard 2D and 3D control tasks---TwoRoom, PushT, Reacher, and Cube---against state-of-the-art JEPA-style world models. Despite using substantially more restricted dynamics, our models match or improve success rate across the tasks. At the same time, exploiting the structured, differentiable dynamics reduces planning time by approximately three orders of magnitude, reducing several minutes of planning to seconds. This gain in efficiency enables evaluation on more demanding control regimes.

\begin{itemize}\item[(C4)] We extend world models beyond short-horizon offline planning to longer-horizon and real-time control.\end{itemize}

We evaluate progressively longer planning horizons and moving-goal environments that require continuous replanning. Bilinear World Models degrade less as the horizon increases while their planning cost scales substantially better than sampling-based alternatives. Their low control cost also allows the same controller to operate in real time, whereas competing world models require a specialized real-time sampling method. Across the moving-goal tasks, Bilinear World Models achieve substantially higher success, showing that structured latent dynamics enable world models to operate beyond the short-horizon offline regime in which they are typically evaluated.

\begin{figure}[t]
    \centering
    \includegraphics[width=\linewidth]{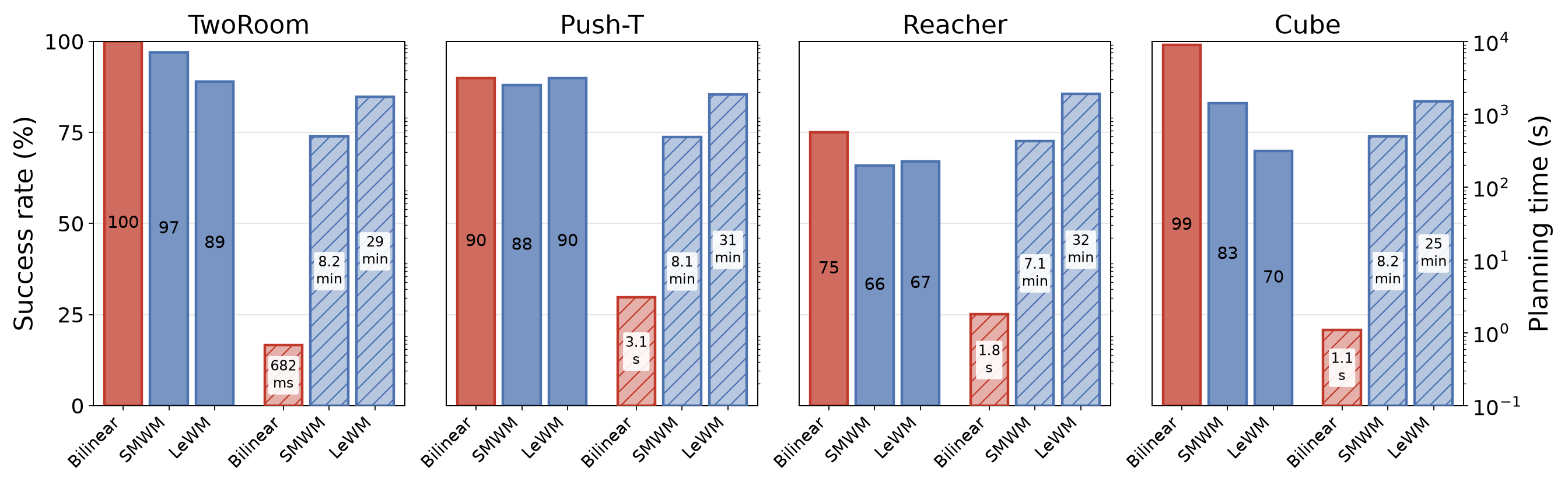}
    \caption{\textbf{Planning success and time across four benchmarks.} Solid bars report success rate over 100 evaluation episodes, each defined by a start--goal pair with a goal offset of 25 steps; hatched bars report the average total wall-clock planning time per episode on a log scale. \textit{Our Bilinear World Model matches or improves success across all environments while reducing planning time by two to three orders of magnitude.}}
    \label{fig:planning_time}
\end{figure}

%% file: 01_introduction/b_related.tex
\section{Related Work}

\paragraph{Latent World models.} Given a dataset of observed transitions generated by an unknown dynamical system, world models aim to jointly learn a latent representation and latent dynamics that capture the behavior of the underlying system~\citep{wm}. Initial approaches, such as PlaNet~\citep{planet} and Dreamer~\citep{hafner2020dreamcontrollearningbehaviors}, relied on encoder--decoder architectures, where pixel reconstruction losses encouraged representations that retained information about the environment. More recently, Joint-Embedding Predictive Architectures (JEPAs)~\citep{assran2023selfsupervisedlearningimagesjointembedding, chen2026vljepajointembeddingpredictive} have provided an alternative that learns latent representations without relying on reconstruction losses. Despite these differences, both paradigms typically model the latent dynamics with flexible neural network predictors.We propose to constrain the latent dynamics through a bilinear parameterization and learn the representation under this structural restriction.


\paragraph{JEPA world models and representation collapse.} A challenge in jointly learning flexible latent representations and latent dynamics is the existence of trivial solutions in which both the encoder and the dynamics predictor collapse to a constant representation~\citep{balestriero2025lejepaprovablescalableselfsupervised}. This phenomenon, known as \textit{representation collapse}, is a central limitation of JEPA-based world models. Early approaches sought to prevent collapse by regularizing the distribution of the latent representations~\citep{leworld,vjepa}. More recently, Sensorimotor World Models (SMWM)~\citep{sensimotor} introduced an auxiliary inverse model that predicts the action from latent transitions. Because the action cannot be recovered from a collapsed representation, this auxiliary objective discourages collapse. In our setting, using a bilinear parameterization of the latent dynamics allows action recoverability to be enforced directly through the model structure, thereby preventing representation collapse by construction.

\paragraph{Bilinear dynamical systems.}
Bilinear systems constitute a classical class of nonlinear control systems in which the dynamics are linear in the state and in the action independently, but contain multiplicative state--action interactions~\citep{1974bilinear}. Their structured form has enabled the development of specialized tools for controllability, stabilization, optimal control, and trajectory optimization, while retaining substantially more expressive dynamics than linear models~\citep{elliott2007bilinear}. More recently, bilinear models have also emerged in data-driven control as structured approximations of nonlinear dynamical systems~\citep{otto2024learning}. In particular, for control-affine systems, the corresponding Koopman generator inherits an affine dependence on the control, naturally giving rise to bilinear dynamics in lifted coordinates~\citep{koopman-adv}. 

\paragraph{Koopman embeddings.}
Koopman operator theory provides a principled framework for representing nonlinear dynamical systems through linear evolution in a lifted, potentially infinite-dimensional space of observables~\citep{koopman,Mezi2005SpectralPO}. This perspective has motivated data-driven methods that approximate the Koopman operator on finite-dimensional observables~\citep{williams2015data}. Deep Koopman methods extend this idea by learning the observables themselves with neural networks, seeking latent representations in which the dynamics are approximately linear~\citep{Lusch_2018, otto}. These methods are often developed from state measurements, while approaches operating on high-dimensional observations commonly employ encoder--decoder objectives~\citep{9636408}. In contrast, we learn structured latent dynamics directly from high-dimensional observations using a JEPA objective, without requiring input reconstruction.

%% file: 02_problem_formulation/input.tex
\section{Bilinear World Models}
\label{sec:embeddings}

\input{02_problem_formulation/a_pf}

%% file: 02_problem_formulation/a_pf.tex
Let $F_s: \mathcal{S} \times \mathcal{A} \to \mathcal{S}$ be an \textit{unknown} discrete-time dynamical system, where a current state $s_t \in \mathcal{S} \subseteq \mathbb{R}^n$ and an action $a_t \in \mathcal{A} \subseteq \mathbb{R}^{m}$ determine the state evolution
$s_{t+1} = F_s(s_t,a_t)$.
In this work, we seek an encoder $\phi:\mathcal{S}\to\mathcal{Z}\subseteq\mathbb{R}^{d}$ that maps each state $s\in\mathcal{S}$ to a latent representation $z=\phi(s)\in\mathcal{Z}$ in which the unknown dynamics $F_s$ admit a simpler structured representation. In particular, we begin from latent dynamics with a bilinear parameterization,
\begin{align}
\label{eq:embedding}
    z_{t+1} = Az_t +  M(z_t) a_t, \quad  M(z_t):= B+Cz_t\footnotemark{}.
\end{align}

\footnotetext{We note that $C \in \mathbb{R}^{d\times m \times d}$ collects $d$ matrices $C_1,\dots,C_d$, with $Cz := \sum_{i=1}^d C_i z_i \in \mathbb{R}^{d\times m}$.}

Restricting the structure of the latent dynamics shifts the modeling burden onto the encoder, which must learn a representation in which the prescribed dynamics are sufficiently expressive. In particular, if $ M(z_t)^\top M(z_t)$ is positive definite, the action that generated a latent transition can be recovered by solving \eqref{eq:embedding} for the action,
\begin{align}
\label{eq:action}
a_t=  M(z_t)^\dagger(z_{t+1}-Az_t), \quad  M(z)^\dagger:=
\big( M(z)^{\top} M(z)\big)^{-1} M(z)^\top.
\end{align}
%
Encouraging recoverability of the action from the state transition has previously been proposed as a mechanism for preventing \textit{representation collapse}~\citep{sensimotor}. Unlike SMWM, which learns an unconstrained inverse network, our dynamics admit a closed-form inverse whenever the structured action map is full rank. A convenient way to enforce this condition in practice is to normalize the matrix $M(z_t)$. In practice, we adopt the following normalization,
\begin{align}
\label{eq:normalized_embedding} \quad
    M(z)=Q(z)R,
    \qquad
    Q(z)=\mathrm{CholQR}\!\left( B+Cz_t\right),
\end{align}
where $Q(z)\in\mathbb{R}^{d\times m}$ is the orthogonal factor obtained through the Cholesky--QR decomposition and $R\in\mathbb{R}^{m\times m}$ is a learned state-independent matrix. We emphasize that the resulting dynamics are no longer strictly bilinear, since the normalization introduces a nonlinear dependence on $z$. However, their state dependence remains generated by the bilinear parameterization, which is subsequently normalized before entering the dynamics.


Two practical considerations must be addressed before formulating the problem. First, optimizing the encoder over the unrestricted space of functions $\phi:\mathcal{S}\to\mathcal{Z}$ is intractable. We therefore restrict it to a finitely parameterized hypothesis class $\phi_\theta\in\mathcal{H}$, with $\theta\in\Theta\subseteq\mathbb{R}^p$. Second, the underlying dynamical system $F_s$ is unknown and is observed only through a finite collection of transitions
$\Omega=\{(s_t,a_t,s_{t+1})_i\}_{i=1}^N$.
Accordingly, we propose to solve the following optimization problem:
%

\begin{equation}
\begin{aligned}
\underset{\phi_\theta,A,M}{\mathrm{minimize}}\quad
&\frac{1}{N}\sum_t
\Big\|
\phi_\theta(s_{t+1})
-
A\phi_\theta(s_t)
-
M\big(\phi(s_t)\big)\,a_t
\Big\|^2 \\
&+
\lambda\frac{1}{N}\sum_t
\Bigg\|
a_t
-
\Big[
M\big(\phi(s_t)\big)\Big]^\dagger
\Big(
\phi_\theta(s_{t+1})
-
A\phi_\theta(s_t)
\Big)
\Bigg\|^2.
\end{aligned}
\label{eq:problem}
\tag{$\textup{P}$}
\end{equation}
where $M(z_t)$ is a parametrized linear function of $z_t$ as in \eqref{eq:normalized_embedding}. The first term shapes the latent space to follow parametrized bilinear dynamics: the encodings of two consecutive states are encouraged to be reachable from one another through an action $a_t$. The second term uses the inverse formula in \eqref{eq:action} to pull this action toward the one in the expert trajectory. Since expert actions are unique, this aligns with requiring $M(z_t)$ to be positive definite.

The problem formulation in~\eqref{eq:problem} jointly learns a latent representation of the environment and its dynamics, and can therefore be viewed as a world-model learning problem. Since the model is trained entirely in latent space, without reconstructing the original observations through a decoder, it follows a JEPA-style formulation.
However, unlike most latent world-model approaches which parameterize the dynamics with an unconstrained neural network~\citep{sensimotor,leworld}, we explicitly restrict the latent dynamics to follow a bilinear parameterization. We remark that the QR decomposition is differentiable, and therefore gradients can be backpropagated through this normalization. Implementation details for solving~\ref{eq:problem}, together with an extended discussion of the relationship between the bilinear and normalized dynamics, are provided in Appendix~\ref{app:sec:problem}.

%% file: 03_koopman_theory/input.tex
\section{Existence of Action-Recoverable Bilinear Embeddings}
\label{sec:manifold}

\input{03_koopman_theory/a_tangent}

%% file: 03_koopman_theory/a_tangent.tex
Restricting the latent dynamics to a bilinear form raises the question of whether such representations are sufficiently expressive. In this section, we characterize conditions under which a nonlinear controlled system admits a possibly infinite-dimensional bilinear embedding. Building on Koopman realizations of control-affine systems, we further show that pointwise independent action directions yield a full-rank latent action map, and therefore action recoverability.

We first state the Koopman linear representation theorem, which ensures that for any dynamical system $F$ and a fixed action $a \in \mathcal{A}$, there exists an embedding where the resulting non actuated dynamics are linear in $z$.

\begin{lemma}[Universal Koopman linear representation]
\label{thm:universal_koopman_representation}
Let $F:\mathcal S\times\mathcal A\to\mathcal S$ be a deterministic continuous-time
controlled dynamical system. For a fixed action $a\in\mathcal A$, there exist a
(possibly infinite-dimensional) vector space $\mathcal Z$, an injective embedding
$E:\mathcal S\to\mathcal Z$, and a linear operator $K_a:\mathcal Z\to\mathcal Z$
such that $z=E(s)$ evolves according to
\begin{equation}
\label{eq:koopman_embedding_dynamics}
\dot z = K_a z.
\end{equation}
\end{lemma}
\begin{proof}
This is the classical Koopman construction; see for example
\cite[Sec.~1]{bamieh2022shortintroductionkoopmanrepresentation}.
\end{proof}

Lemma \ref{thm:universal_koopman_representation} ensures the existence of a transformation into an embedding space where dynamics are linear in $z$. 
However, we note that the linear operator $K_a$ is not necessarily affine on $a$. To guarantee that we require the original system to be action affine and a mild structural requirement. Moreover, to ensure injectiveness of the control operator we require that the effect actions have in the system for a particular state $s$ are nondegenerate.
 
\begin{assumption}[Affinity in the actions and Lipchitzness]
\label{ass:action_affine}
The dynamical system $F_s$ is \emph{affine in the actions}. That is, there exist functions
$f_0,f_1,\ldots,f_m$ such that
\begin{equation}
\label{eq:action_affine}
    F_s(s,a)
    =
    f_0(s)
    +
    \sum_{j=1}^{m} a_j f_j(s).
\end{equation}
Moreover, the system is Lipchitz if the functions $f_0,f_1,\ldots,f_m$ are Lipchitz. 
\end{assumption}

\begin{assumption}[Independent action directions]
\label{ass:independent_actions}
The actions induce linearly independent directions in the system dynamics.

In particular, for a system affine in the actions
(Assumption~\ref{ass:action_affine}) that requieres, 
\begin{equation}
\label{eq:independent_action_directions}
    \operatorname{rank}
    \begin{bmatrix}
        f_1(s) & \cdots & f_m(s)
    \end{bmatrix}
    = m,
    \qquad
    \forall s\in\mathcal S.
\end{equation}
\end{assumption}
Affinity in the actions is a common structure in mechanical systems when the actions correspond to generalized forces or torques, rather than to lower-level actuator signals. 
%
Requiring the dynamics to be Lipchitz continuous requires that nearby states can not induce arbitrarily different dynamics. Nontheless, it does not require differentiability and therefore still allows for nonsmooth behaviors, such as those arising from contact interactions.
Assumption~\ref{ass:independent_actions} requires distinct actions to induce  distinct direction of motion, at least locally. In particular,
no action direction can be reproduced as a linear combination of the
effects induced by the remaining actions.
Then, whenever the system satisfy this conditions, there exist a transformation were the embeddings follow bilinear dynamics. 

\begin{proposition}[Action-recoverable bilinear Koopman embedding]
\label{prop:bilinear_embedding}
Let Assumptions~\ref{ass:action_affine} and
\ref{ass:independent_actions} hold. Then there exist a possibly
infinite-dimensional vector space $\mathcal Z$, an injective embedding
$E:\mathcal S\rightarrow\mathcal Z$, and linear operators
$
A:\mathcal Z\to\mathcal Z$,
$B:\mathbb R^m\to\mathcal Z$ and 
$C:\mathcal Z\to\mathcal L(\mathbb R^m,\mathcal Z),
$
such that, for $z=E(s)$,
\begin{equation}
    \dot z
    =
    Az + M(z)a,
    \qquad
    M(z)=B+Cz.
\end{equation}
Moreover, the latent action map $M(z)$ is injective for every
$z\in E(\mathcal S)$. Equivalently,
\begin{equation}
    M(z)^TM(z)\succ 0,
    \qquad
    \forall z\in E(\mathcal S).
\end{equation}
\end{proposition}

\begin{proof}
    See Appendix~\ref{apx:proof_bilinear_thm}.
\end{proof}

Proposition~\ref{prop:bilinear_embedding} shows that the two structural properties required by our model---bilinear latent dynamics and action recoverability---can hold simultaneously for a broad class of nonlinear controlled systems. Existing Koopman results provide possibly infinite-dimensional bilinear realizations for control-affine systems, while the additional independence assumption on the physical action directions guarantees that the corresponding latent action map is injective, and therefore that actions can be recovered from latent transitions. The result is only an expressiveness argument, as it applies to exact continuous-time, possibly infinite-dimensional representations, whereas Problem~\ref{eq:problem} learns a finite-dimensional, discrete-time, normalized approximation. Prior work nevertheless provides empirical evidence that finite-dimensional neural embeddings can approximate useful Koopman representations across a range of dynamical systems~\citep{Lusch_2018,otto2024learning,9636408}.

%% file: 04_control/input.tex
\section{Control in Normalized Bilinear Systems}
\label{sec:control}

\input{04_control/a_control}

%% file: 04_control/a_control.tex
A useful world model should not only predict the evolution of a dynamical system, but also enable efficient planning and control between states. In particular, given an initial state $z_0$ and a goal $z^\star$, we seek a sequence of actions $\mathbf{a}=(a_0,\dots,a_{H-1})$ that drives the system toward $z^\star$ over a planning horizon $H$. The structured dynamics introduced in~\eqref{eq:normalized_embedding} make this optimization particularly efficient.

To formalize this objective, let $z_H(\mathbf{a})$ denote the endpoint obtained by recursively applying the learned dynamics under the action sequence $\mathbf{a}$. We optimize
\begin{align}
\label{eq:planning-objective}
\mathcal L(\mathbf{a})
=
\frac{1}{2}\|z_H(\mathbf{a})-z^\star\|^2
+
\frac{c}{2}\|\mathbf{a}\|^2.
\end{align}

The first term measures the error in reaching the goal, while the second regularizes the magnitude of the actions. The objective in~\eqref{eq:planning-objective} does not rely on the particular structure of the dynamics and can therefore be optimized with general black-box world models using sampling-based planners such as the Cross-Entropy Method (CEM). CEM repeatedly samples candidate action sequences, evaluates their costs through model rollouts, and updates the sampling distribution toward the best-performing candidates. Because this procedure does not exploit sensitivity information, it can require a large number of model evaluations, making planning computationally expensive, particularly as the planning horizon or action dimension increases.

The advantage of the normalized bilinear parameterization is that the learned dynamics remain explicitly differentiable with respect to both states and actions, allowing the same objective to be optimized using sensitivity information. For the dynamics
$z_{t+1}=Az_t+M(z_t)a_t$, define the state Jacobian
$F_t:=\partial z_{t+1}/\partial z_t
=A+\partial(M(z_t)a_t)/\partial z_t$.
The Jacobian of the final state with respect to the complete action sequence then takes the form
\begin{align}
\label{eq:control-jacobian}
\mathcal J(\mathbf{a})
:=
\frac{\partial z_H(\mathbf{a})}{\partial \mathbf{a}}
=
\left[
\left(\prod_{j=k+1}^{H-1}F_j\right)M(z_k)
\right]_{k=0}^{H-1}.
\end{align}

The Jacobian in~\eqref{eq:control-jacobian} characterizes how local perturbations of each action affect the final predicted state. Because $M(z)$ is obtained through differentiable operations, including the Cholesky--QR normalization, these sensitivities can be computed efficiently by automatic differentiation. We can therefore use them to construct a local linear approximation of the endpoint as a function of the action sequence. Substituting this approximation into~\eqref{eq:planning-objective} and minimizing the resulting quadratic model gives the update
\begin{align}
\label{eq:gn-control}
\mathbf{a}^{t+1}
=
\mathbf{a}^{t}
-
\left[
\mathcal J(\mathbf{a}^{t})^\top \mathcal J(\mathbf{a}^{t})
+
cI
\right]^{-1}
\left[
\mathcal J(\mathbf{a}^{t})^\top
\big(z_H(\mathbf{a}^{t})-z^\star\big)
+
c\mathbf{a}^{t}
\right].
\end{align}

Applying~\eqref{eq:gn-control} iteratively yields a Gauss--Newton planner~\cite{gn}. In practice, several inexpensive modifications improve its robustness, including initializing the action sequence from a small set of candidates, using backtracking line search to avoid overshooting, and adding damping to the linear system for numerical stability. We additionally exploit the structure of the action space and compute the required sensitivities efficiently in batch. Complete implementation details are provided in Appendix~\ref{app:sec:experients}.

%% file: 05_experiments/input.tex
\section{Numerical Experiments}
\label{sec:experiments}

\input{05_experiments/exps}

%% file: 05_experiments/exps.tex
In this section,we empirically validate the preceding claims by studying the representations learned under the normalized bilinear parameterization and evaluating their implications for control. We first analyze the geometry of the learned latent space, then compare control performance and planning efficiency against existing JEPA-style world models, and finally evaluate longer-horizon and real-time control regimes.

\textbf{Tasks.} We evaluate our framework on the standard LeWM benchmark suite~\citep{leworld}, which includes \textit{Two Room} for 2D navigation, \textit{Cube} for 3D manipulation, \textit{Reacher} for continuous control, and \textit{PushT} for contact-rich object interaction. We measure performance using the \textit{Success Rate}, defined as the percentage of 100 start--goal tasks successfully completed within the allowed action budget. We additionally report the planning time $t$, measured as the computation required by the controller to obtain the action sequence for each task.

\textbf{Benchmarks.} We compare against LeWM~\citep{leworld} and Sensorimotor~\citep{sensimotor}. We focus on latent world models rather than task-specific controllers because our goal is to evaluate whether structuring the latent dynamics preserves control performance while reducing planning cost within this flexible framework, rather than to develop specialized controllers for the four environments considered here. We select LeWM and Sensorimotor because they are the strongest world-model baselines we identified in the literature for these benchmarks.

\textbf{Architecture.} We adopt the encoder used in LeWM~\citep{leworld}, a ViT-Tiny with 5M parameters that maps each state to a 192-dimensional latent representation. The predictor follows~\eqref{eq:normalized_embedding} and is therefore parameterized only by the four matrices $A$, $B$, $C$ and $R$.

\textbf{Controllers.}
For LeWM and Sensorimotor, the main benchmark and long-horizon experiments use CEM~\citep{cem}, matching the planner used in their original evaluations. For our Bilinear World Model, we use the Gauss--Newton planner introduced in Section~\ref{sec:control}. We additionally evaluate CEM, iCEM~\citep{icem}, and Gauss--Newton across all three world models in Appendix~\ref{app:sec:controller_comparison} to separate the effect of model structure from that of the controller. In the real-time experiments proposed in this section, CEM is too slow to meet the prescribed control budget, and we therefore use iCEM for the neural-dynamics baselines while retaining the same Gauss--Newton controller for our method.

\begin{figure}[t]
    \centering
    \begin{subfigure}{0.195\textwidth}
        \centering
        \includegraphics[width=\linewidth]{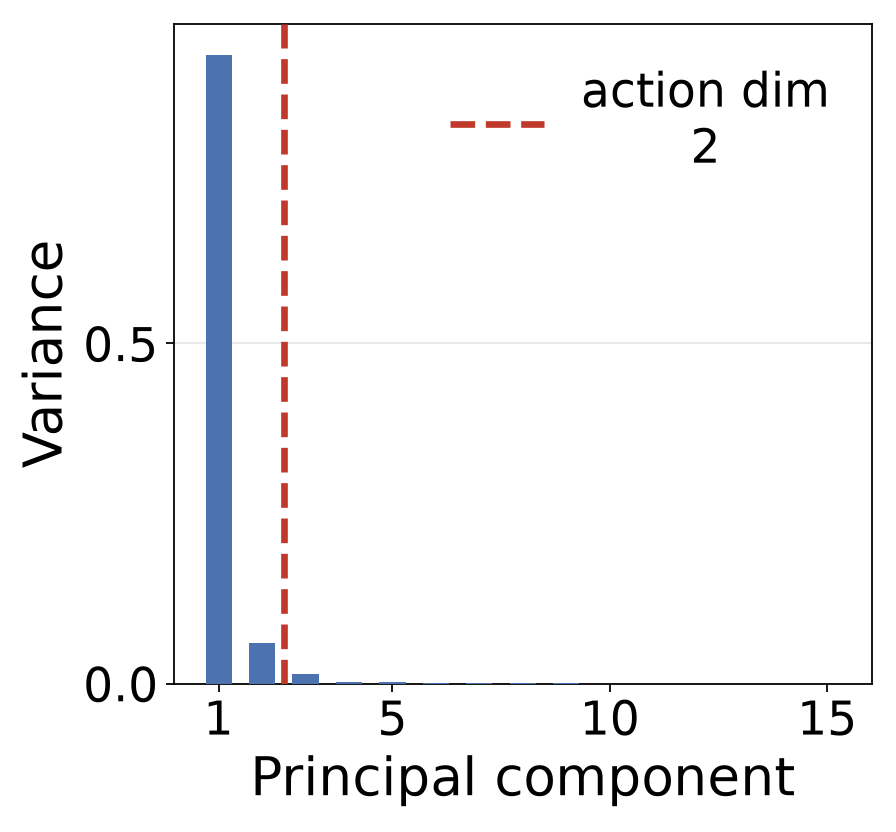}\\[0.5em]
        \includegraphics[width=\linewidth]{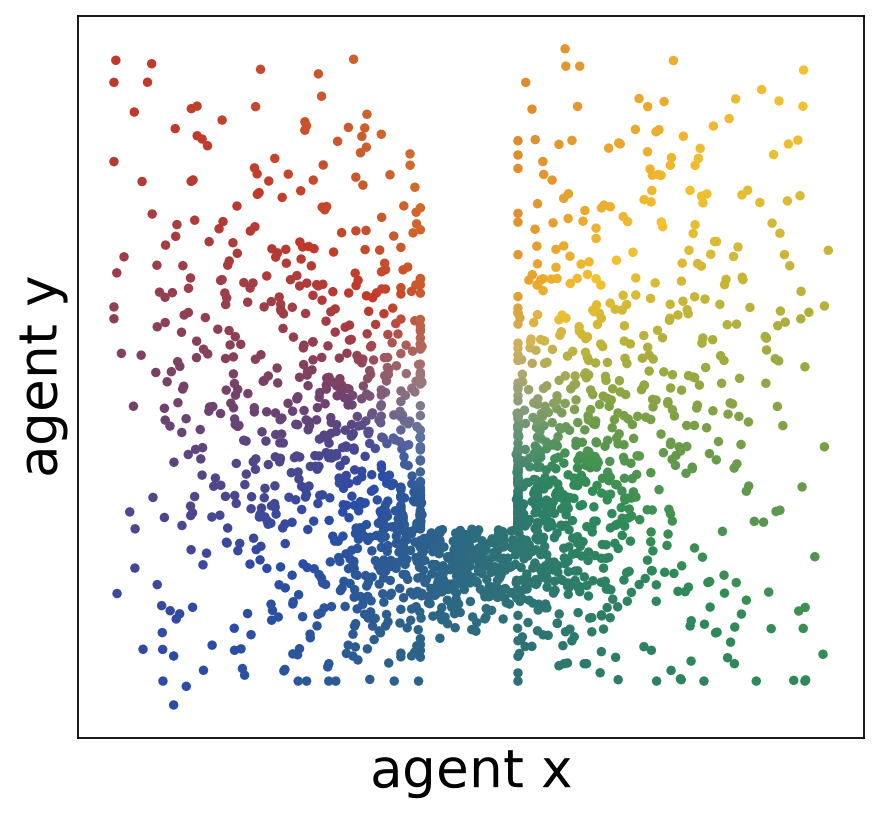}\\[0.5em]
        \includegraphics[width=\linewidth]{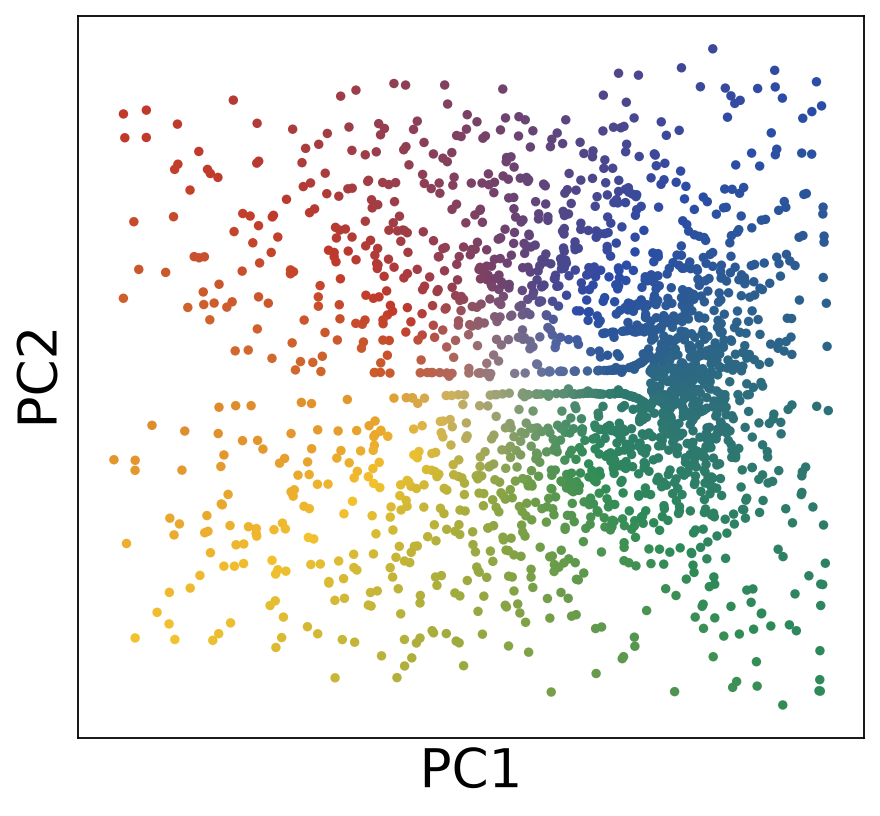}\\[0.5em]
        \includegraphics[width=\linewidth]{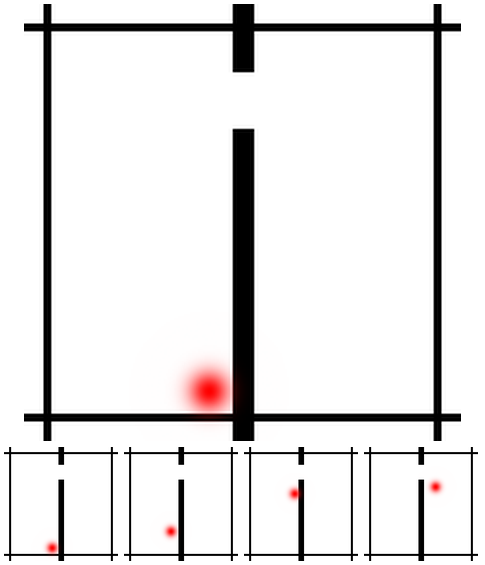}
        \caption{TwoRoom}
    \end{subfigure}
    \hfill
    \begin{subfigure}{0.195\textwidth}
        \centering
        \includegraphics[width=\linewidth]{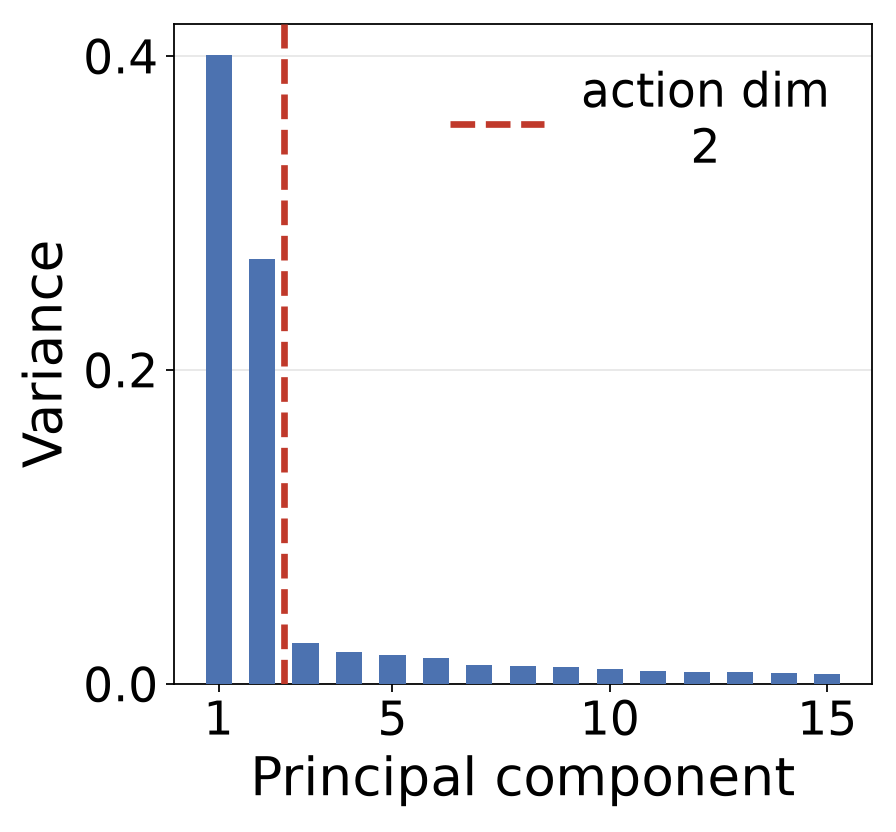}\\[0.5em]
        \includegraphics[width=\linewidth]{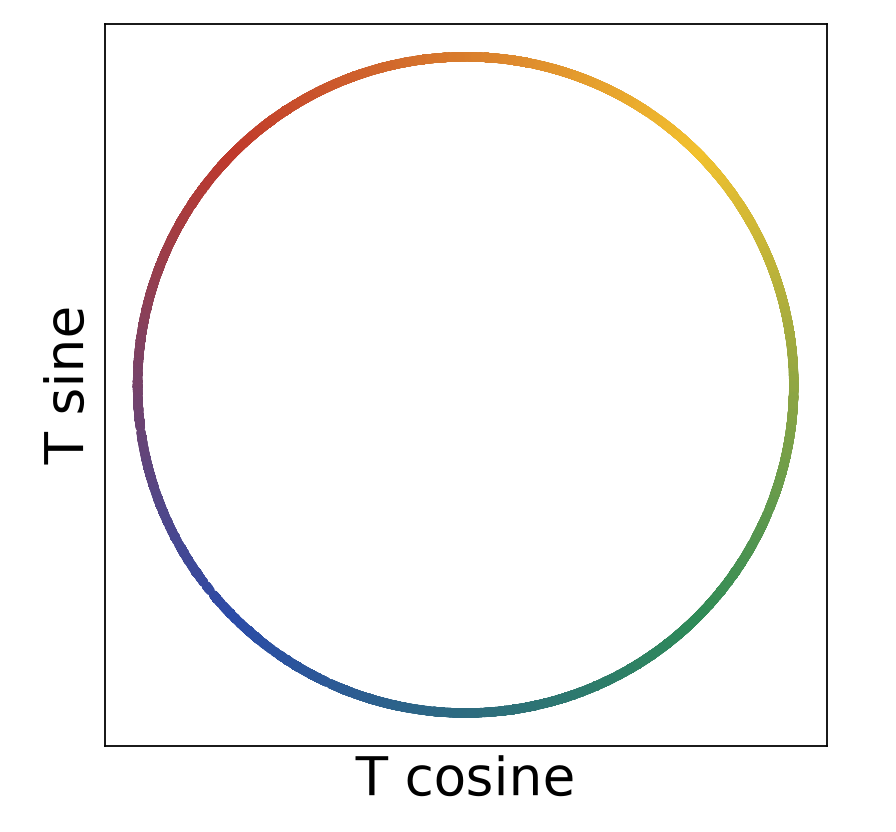}\\[0.5em]
        \includegraphics[width=\linewidth]{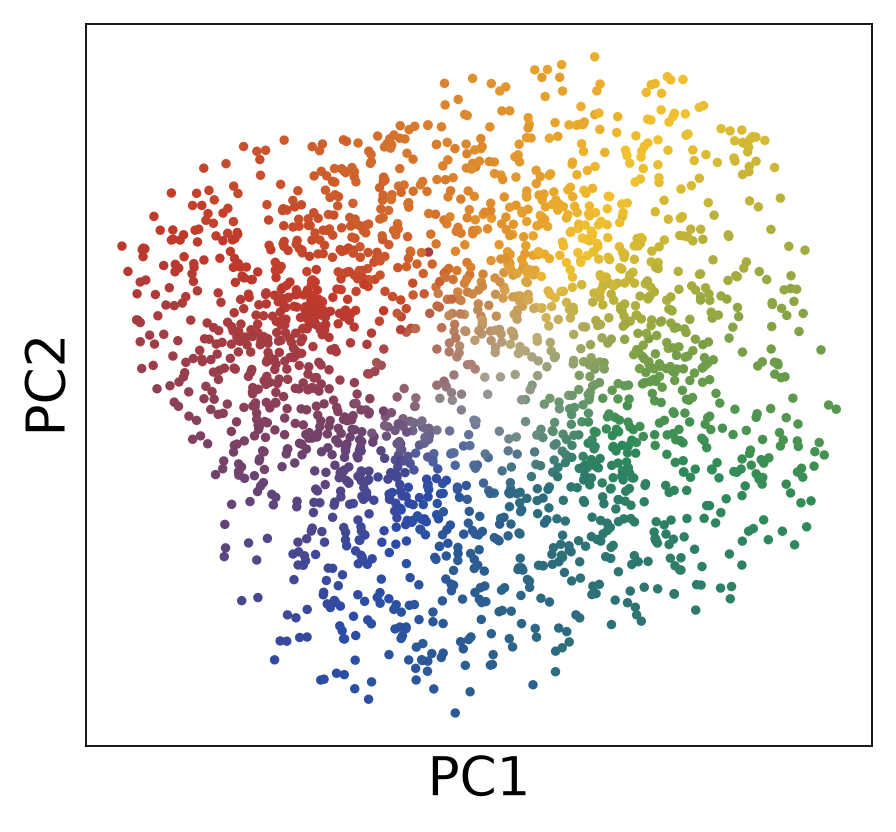}\\[0.5em]
        \includegraphics[width=\linewidth]{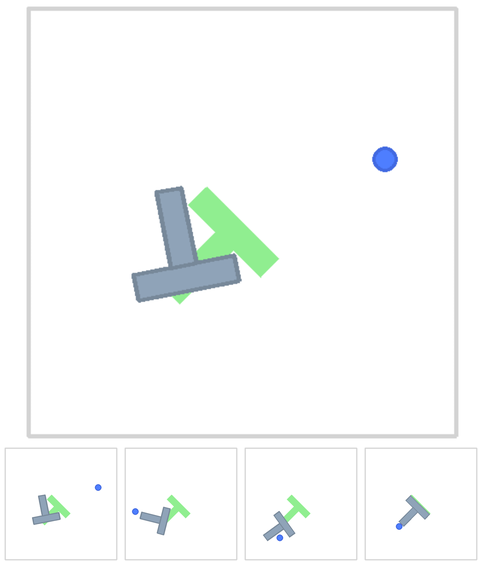}
        \caption{PushT}
    \end{subfigure}
    \hfill
    \begin{subfigure}{0.195\textwidth}
        \centering
        \includegraphics[width=\linewidth]{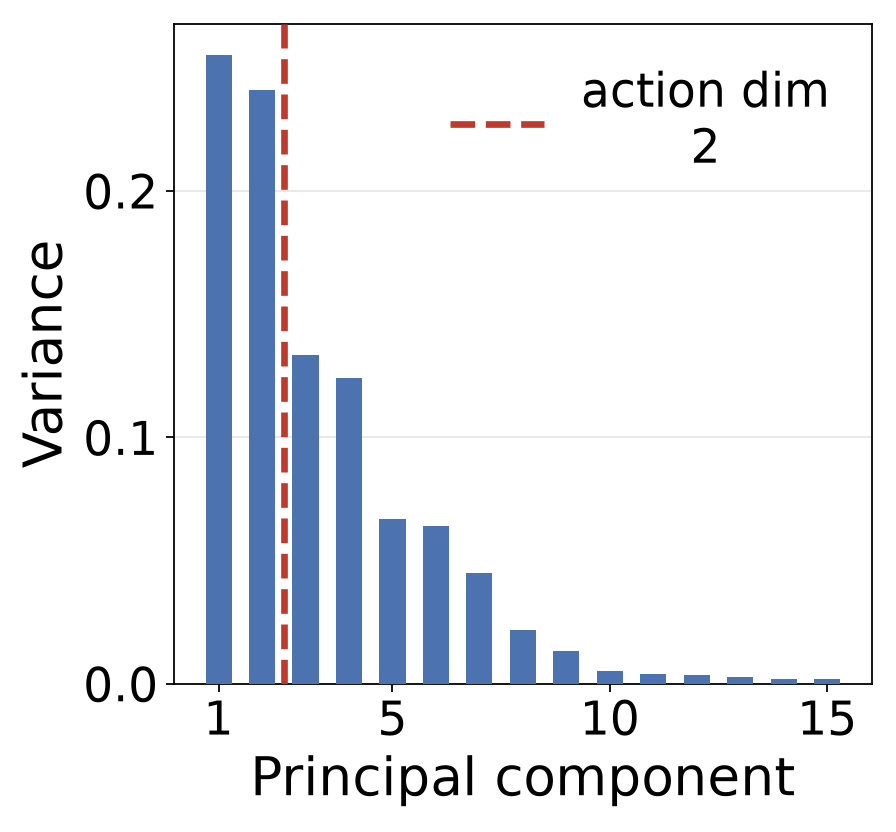}\\[0.5em]
        \includegraphics[width=\linewidth]{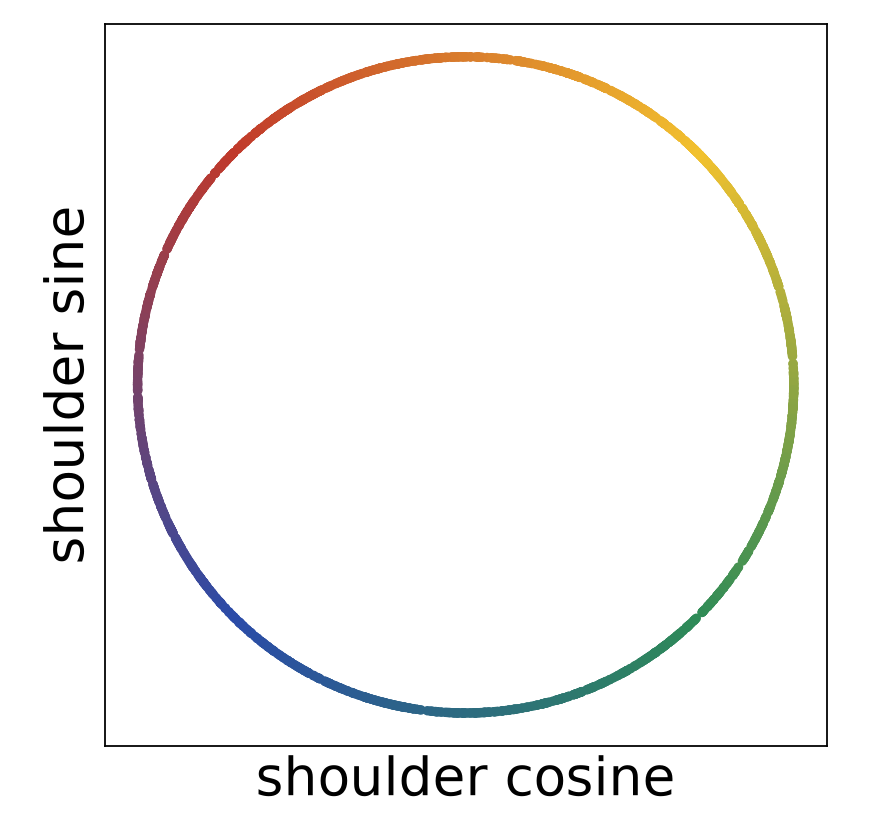}\\[0.5em]
        \includegraphics[width=\linewidth]{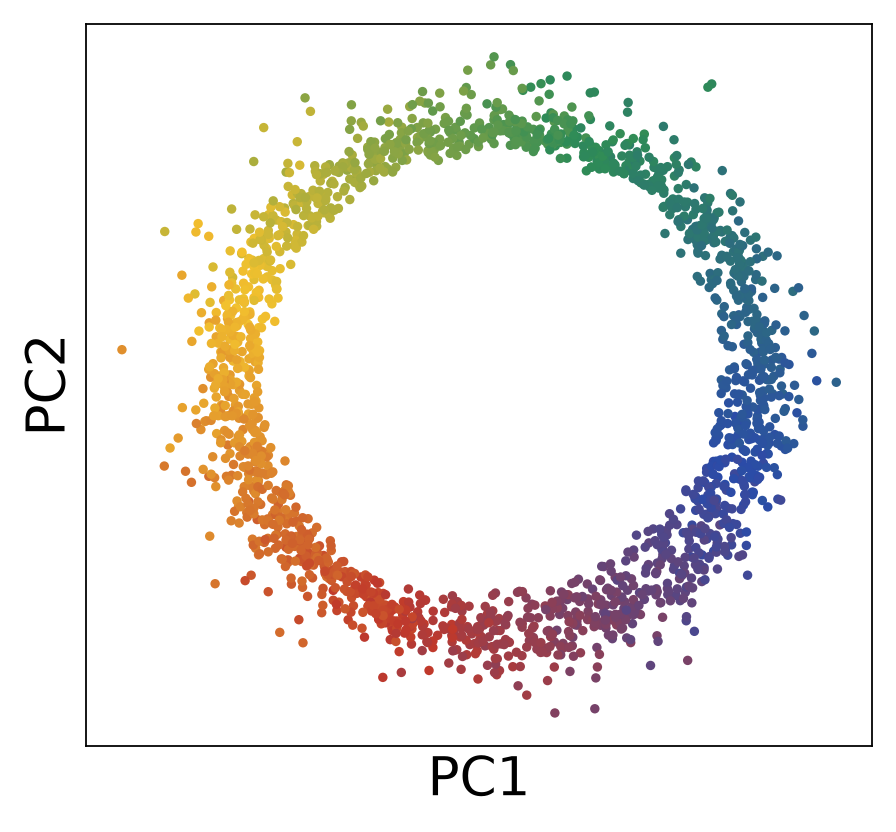}\\[0.5em]
        \includegraphics[width=\linewidth]{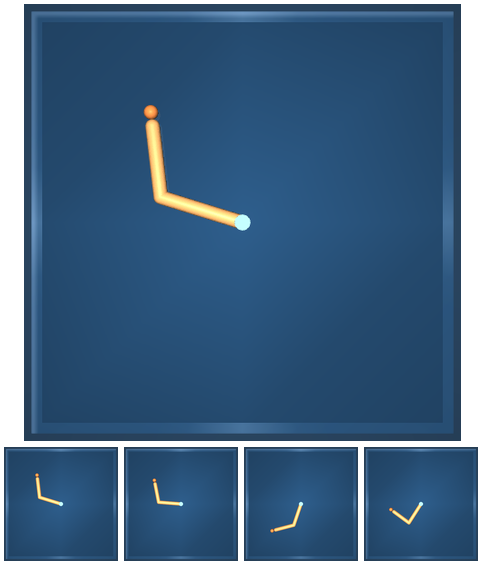}
        \caption{Reacher}
    \end{subfigure}
    \hfill
    \begin{subfigure}{0.195\textwidth}
        \centering
        \includegraphics[width=\linewidth]{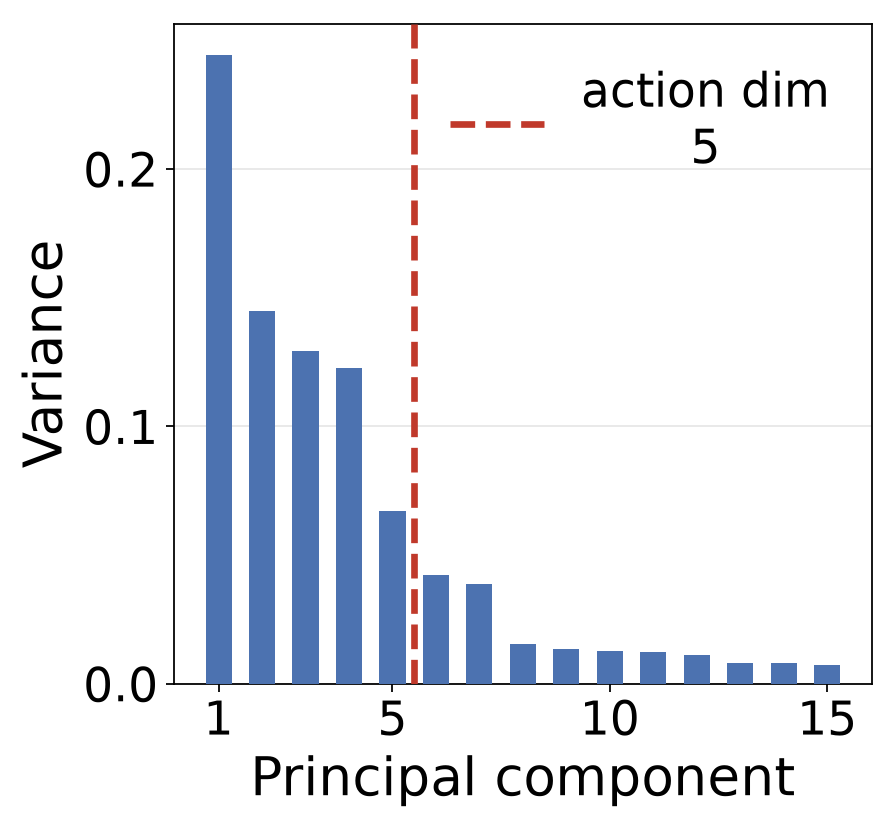}\\[0.5em]
        \includegraphics[width=\linewidth]{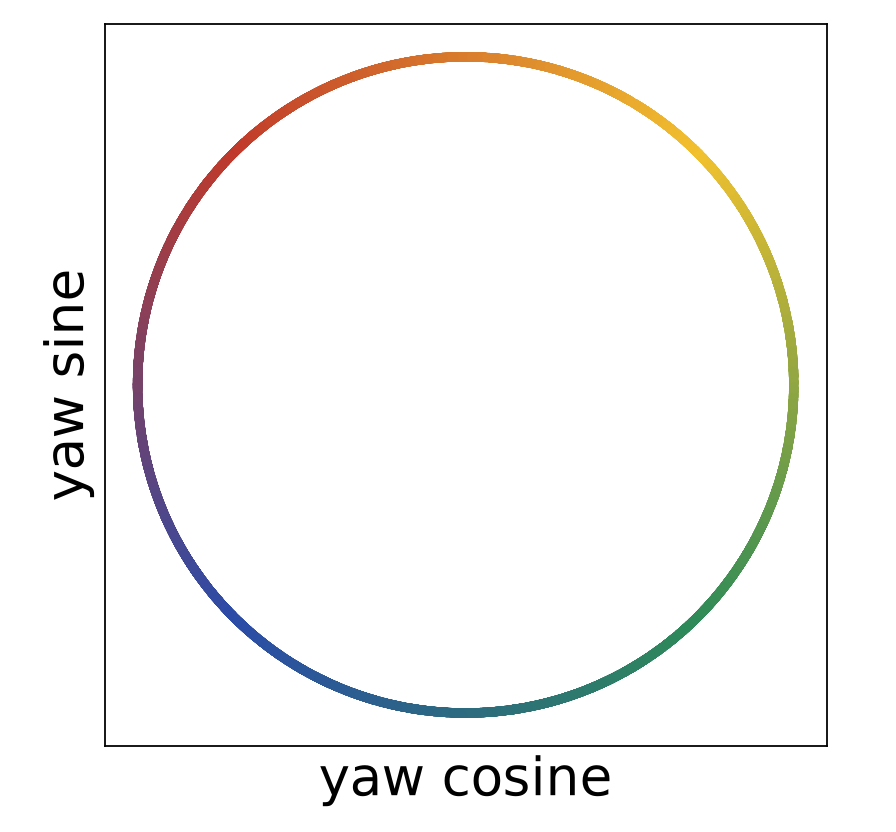}\\[0.5em]
        \includegraphics[width=\linewidth]{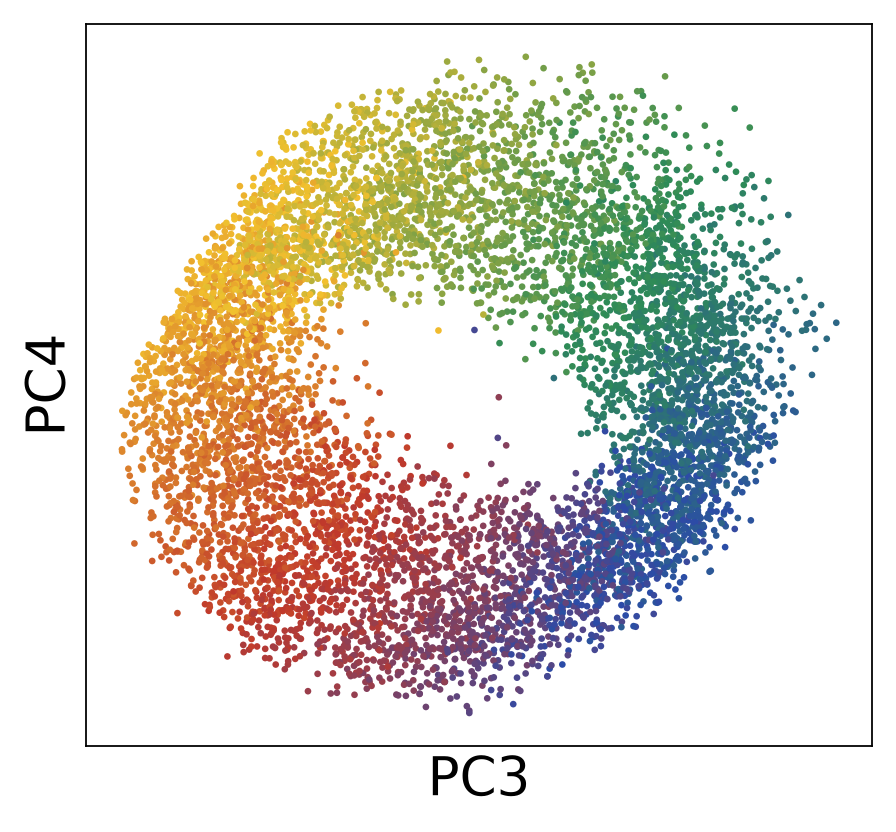}\\[0.5em]
        \includegraphics[width=\linewidth]{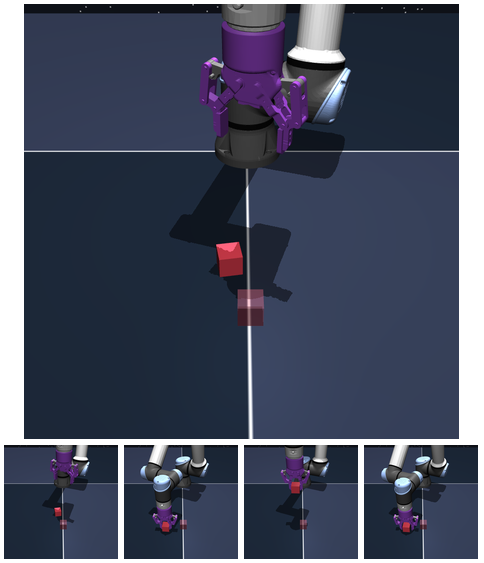}
        \caption{OG-Cube}
    \end{subfigure}
\caption{\textbf{Geometry of the learned latent representations.}
(Top) Variance explained by each principal component of the latent embeddings, with the action dimension indicated by the red dashed line. (Middle) Comparison between selected physical state variables and the corresponding projections of the latent embeddings onto their principal components. In TwoRoom, PC1 and PC2 are shown against the agent's horizontal and vertical positions. In PushT, Reacher, and Cube, principal-component projections are shown alongside angular state variables represented through their sine and cosine coordinates. (Bottom) Visualizations of the corresponding environments.}
    \label{fig:geometry}
\end{figure}

\subsection{Learned Representations and Control Performance}

\textbf{Bilinear World Models capture controllable degrees of freedom.}
Actions induce changes in the physical state of a system. By explicitly structuring how these same actions induce transitions in latent space, our parameterization encourages the encoder to organize representations according to the physical quantities affected by the actions. Figure~\ref{fig:geometry} shows that this correspondence indeed emerges in the learned representations. In TwoRoom (a), the first two principal components closely recover the agent's physical coordinates, with PC1 capturing its horizontal position and PC2 its vertical position. In PushT (b), Reacher (c), and Cube (d), the dominant components recover relevant angular degrees of freedom, forming circular structures that correspond to sine--cosine representations of the underlying angles. Moreover, across environments, most of the latent variance is concentrated in approximately as many principal components as there are action dimensions, suggesting that the learned manifold is organized around the controllable degrees of freedom of the system.

\textbf{Bilinear World Models achieve high success rates with substantially faster planning.}
The structured dynamics enable controllers to exploit local sensitivity information, resulting in markedly lower planning cost without sacrificing task performance. As shown in Figure~\ref{fig:planning_time}, Bilinear World Models match or improve the success rates of the alternative models across all four environments. Our method solves 100\% of the TwoRoom tasks and 99\% of the Cube tasks, while achieving success rates of 90\% on PushT and 75\% on Reacher. The largest improvement is observed in Cube, where the success rate exceeds that of the strongest baseline by 17 percentage points. The difference in computational cost is even more pronounced: across the benchmarks, our method reduces planning time by approximately three orders of magnitude, turning planning times of several minutes---up to 7--32 minutes for the baselines---into 0.7–3.1 seconds.

\textbf{The computational gains are not solely due to the choice of controller.}
To disentangle the effect of the learned dynamics from that of the planner, we evaluate CEM, iCEM, and Gauss--Newton (GN) on each of the three world models; full results are reported in Appendix~\ref{app:sec:controller_comparison}. GN is particularly effective when paired with Bilinear World Models, where it can exploit the explicit structure of the latent dynamics, achieving high success at substantially lower planning cost. Applying GN to SMWM or LeWorld does not provide the same combination of performance and efficiency, while iCEM reduces computation relative to CEM at the expense of lower success on the more challenging tasks. Based on this comparison, we use CEM for SMWM and LeWorld in all main benchmark and long-horizon experiments, matching the controller used in their original evaluations, and reserve iCEM only for the real-time setting where CEM cannot satisfy the required planning budget.

\textbf{Bilinear World Models enable smoother control trajectories.}
The CEM controllers used in the original baseline implementations optimize primarily for reaching the target and do not explicitly penalize variations in the action sequence. In contrast, the structured dynamics of Bilinear World Models allow such regularization to be incorporated directly into the planning objective, enabling efficiency not only in computation but also in the resulting control trajectories. Figure~\ref{fig:planning} illustrates this effect in TwoRoom, where our method achieves the lowest average jerk, measured from changes in consecutive actions. This difference is also visible in the planned trajectories: while the baseline methods produce irregular paths with abrupt changes in direction, the Bilinear World Model follows a straight trajectory toward the goal.

\begin{figure}[t]
    \centering
    \begin{subfigure}{0.235\textwidth}
        \centering
        \includegraphics[width=\linewidth]{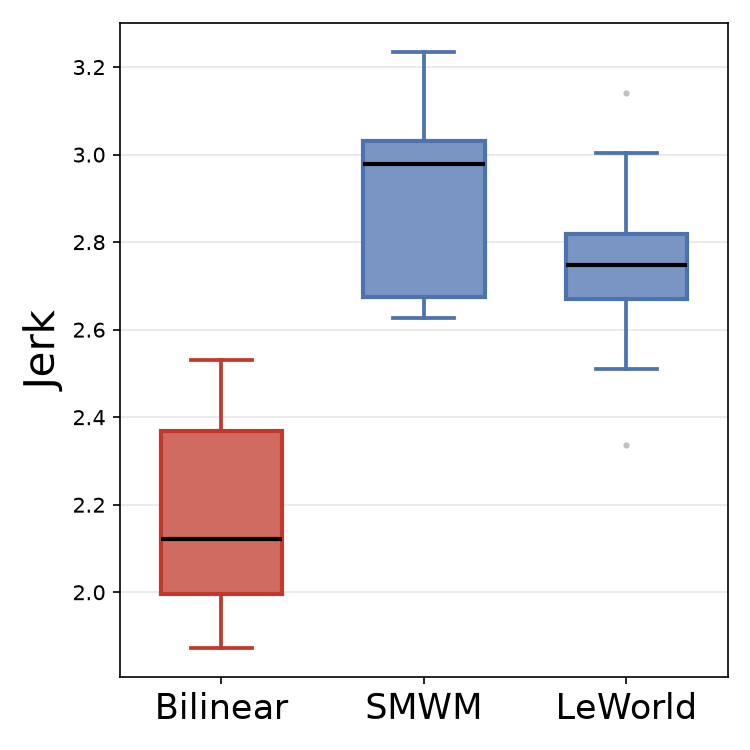}
        \vspace{1mm}
    \end{subfigure}
    \hfill
    \begin{subfigure}{0.235\textwidth}
        \centering
        \includegraphics[width=\linewidth]{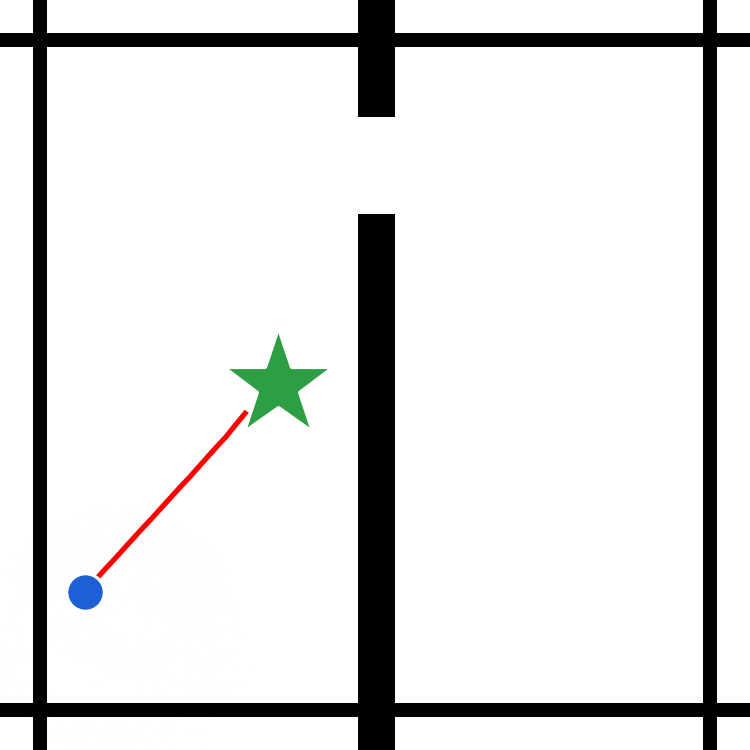}
        \caption{Bilinear}
    \end{subfigure}
    \hfill
    \begin{subfigure}{0.235\textwidth}
        \centering
        \includegraphics[width=\linewidth]{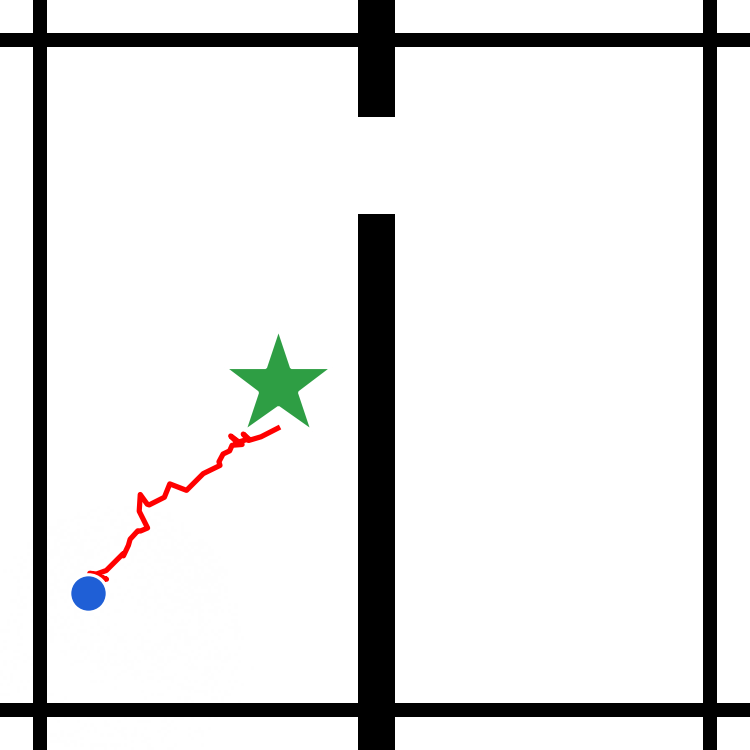}
        \caption{SMWM}
    \end{subfigure}
    \hfill
    \begin{subfigure}{0.235\textwidth}
        \centering
        \includegraphics[width=\linewidth]{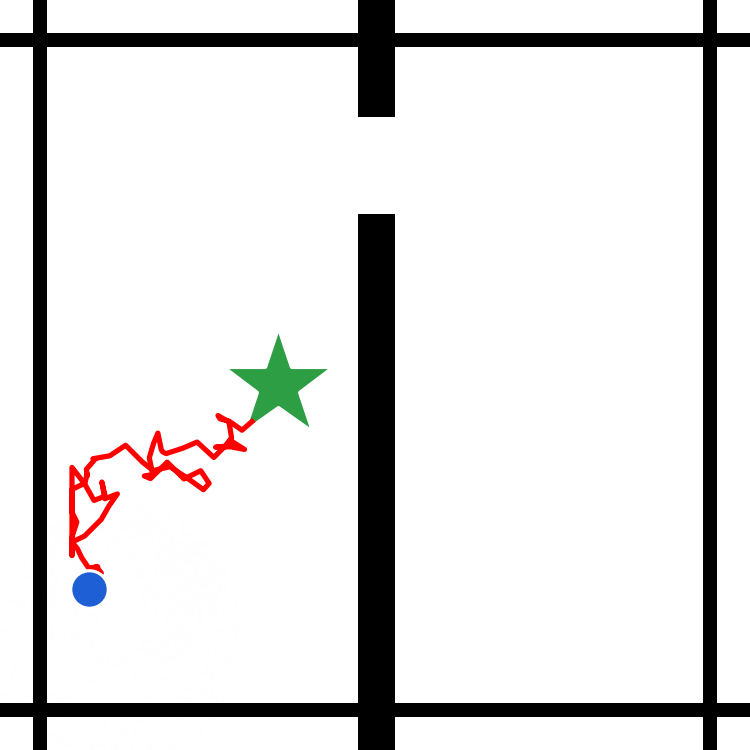}
        \caption{LeWorld}
    \end{subfigure}
\caption{\textbf{Trajectories and action smoothness in TwoRoom.}
(Left) Average jerk, measured from changes between consecutive actions, for each world model. (Right) Representative planned trajectories toward the goal. \textit{Bilinear World Models produce lower-jerk action sequences which result in smoother trajectories.}}
    \label{fig:planning}
\end{figure}

\subsection{Long-Horizon and Real-Time Control}

\textbf{Bilinear World Models scale robustly to longer planning horizons.}
We evaluate control performance as the goal offset increases to 25, 50, and 75 steps, denoted @25, @50, and @75, respectively. As shown in Figure~\ref{fig:horizons}, increasing the goal offset generally makes the control problem more challenging, although its effect varies across environments. Bilinear World Models remain competitive across all evaluated offsets while preserving a substantial computational advantage over the sampling-based alternatives. In TwoRoom, for example, success remains above 75\% at @75, while the alternative models drop to approximately 40\%. A similar advantage is observed in Reacher, where increasing the goal offset has a substantially smaller effect on the Bilinear model. In PushT and Cube, performance varies less consistently with the goal offset, but Bilinear World Models continue to achieve comparable success at substantially lower planning cost. This computational advantage becomes increasingly pronounced for longer-horizon tasks: at @75, LeWorld requires nearly 1.4 hours of planning per episode and SMWM approximately 22 minutes, whereas our method requires only 6.5 seconds.

\begin{figure}[t]
    \centering

    \begin{subfigure}{0.45\textwidth}
        \centering
        \includegraphics[width=\linewidth]{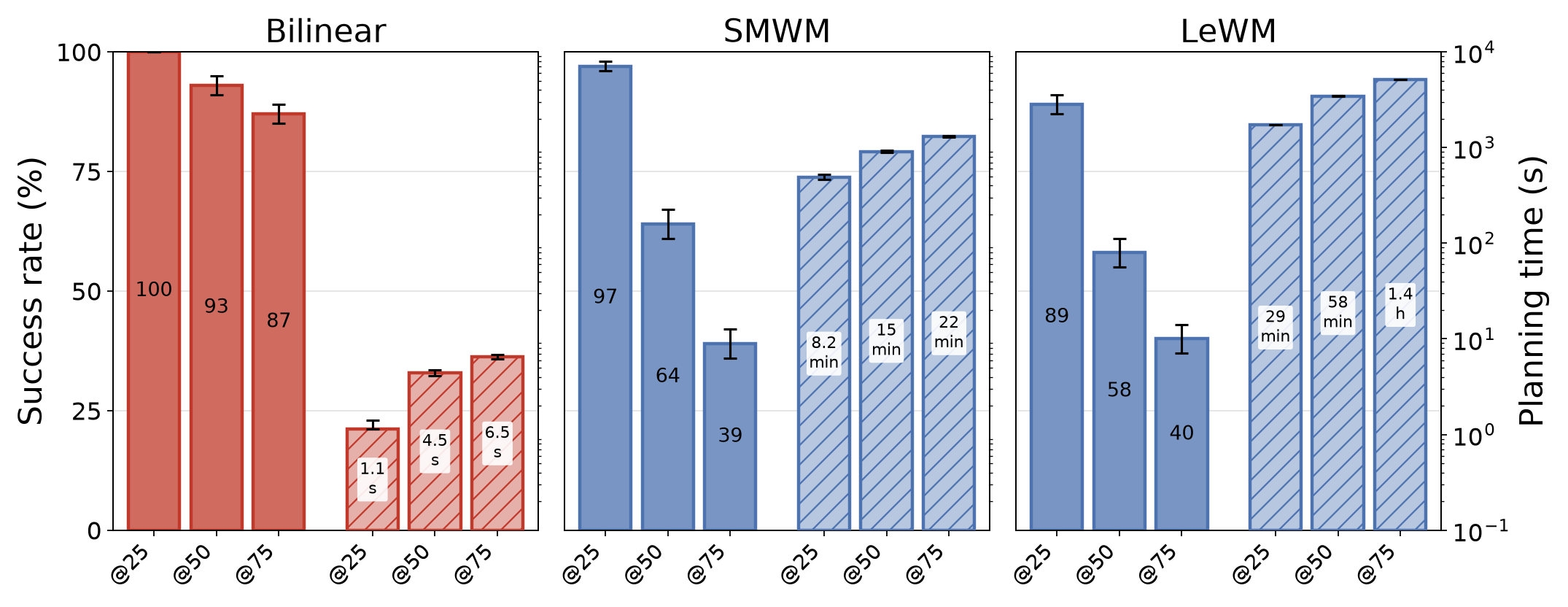}
        \caption{TwoRoom}
    \end{subfigure}
    \hfill
    \begin{subfigure}{0.45\textwidth}
        \centering
        \includegraphics[width=\linewidth]{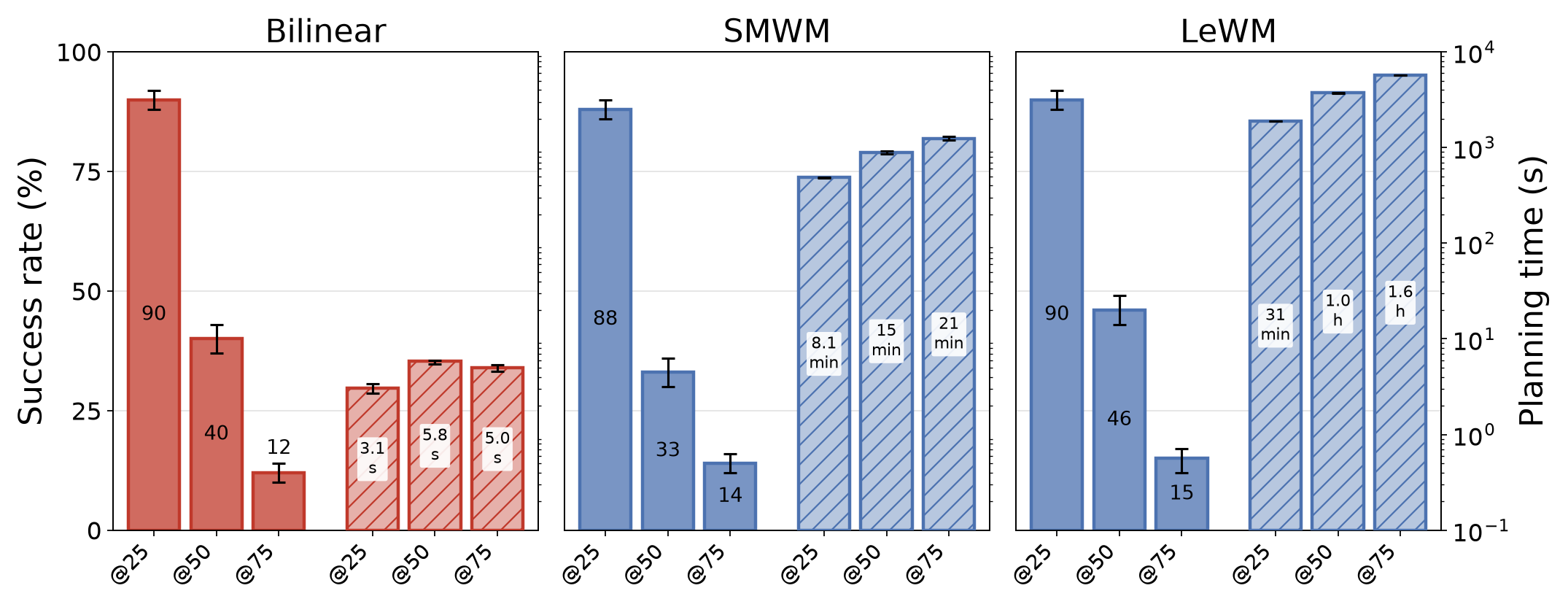}
        \caption{PushT}
    \end{subfigure}

    \vspace{0.5em}

    \begin{subfigure}{0.45\textwidth}
        \centering
        \includegraphics[width=\linewidth]{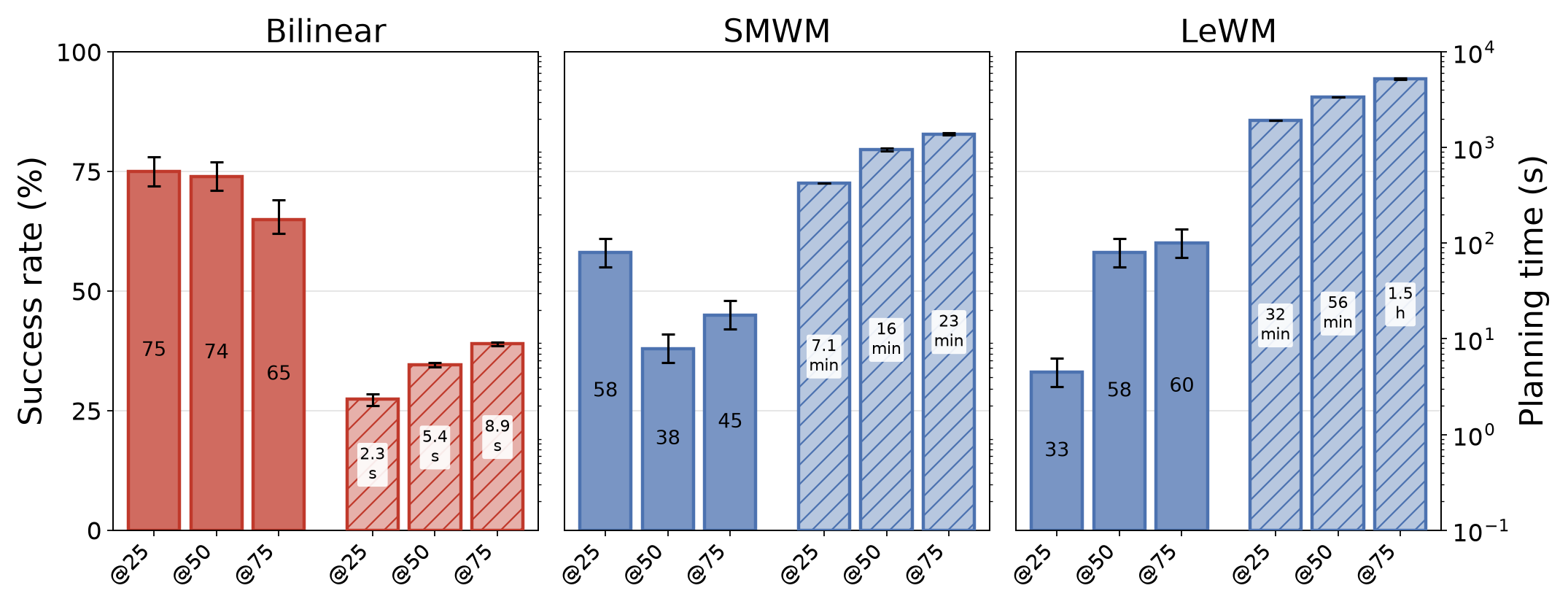}
        \caption{Reacher}
    \end{subfigure}
    \hfill
    \begin{subfigure}{0.45\textwidth}
        \centering
        \includegraphics[width=\linewidth]{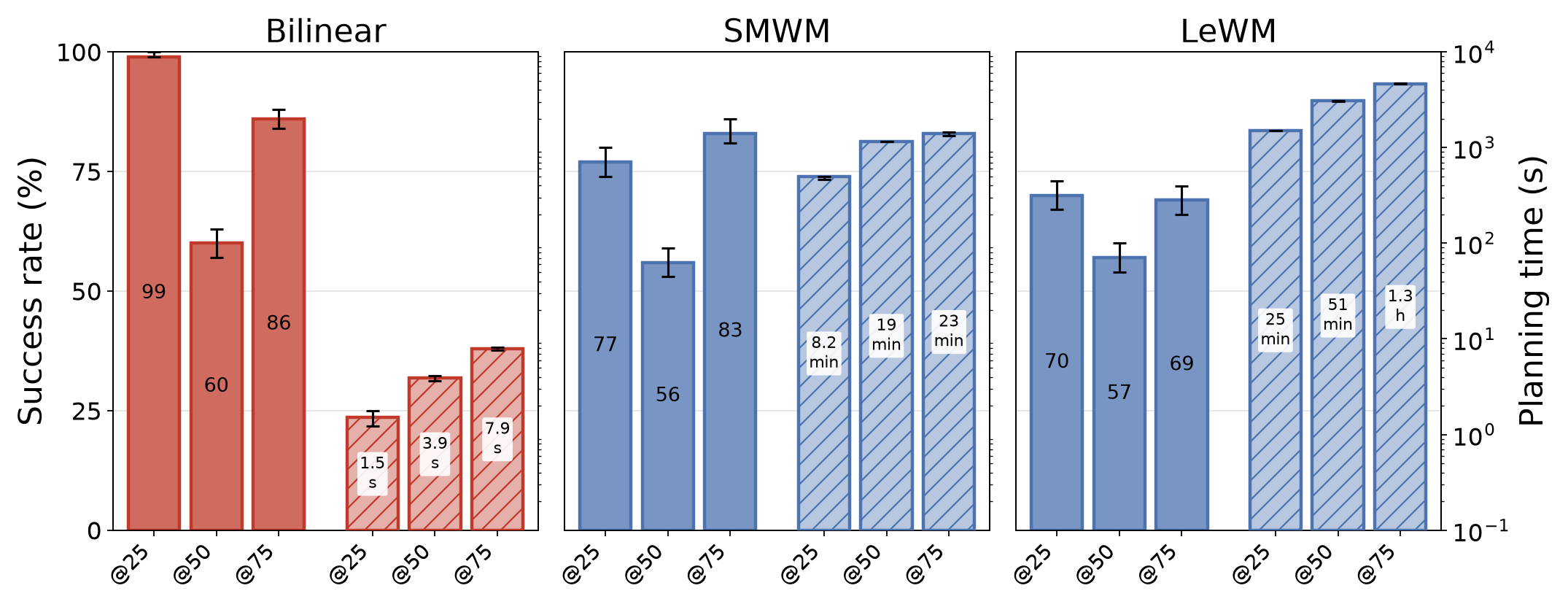}
        \caption{OG-Cube}
    \end{subfigure}

    \caption{\textbf{Performance across increasing task horizons.}
    Each panel reports success rate (solid bars) and planning time (hatched bars) for goal offsets of 25, 50, and 75 steps, denoted as @25, @50, and @75. \textit{Bilinear World Models exhibit a smaller degradation as the horizon increases while maintaining substantially lower planning times.}}
\label{fig:horizons}
\end{figure}

\textbf{Bilinear World Models enable real-time control.}
Fast planning is particularly important in dynamic environments, where the controller must repeatedly replan as the goal changes. We therefore introduce moving-goal variants of the four environments, denoted M-TwoRoom, M-PushT, M-Reacher, and M-Cube, in which the target follows a random walk; implementation details are provided in Appendix~\ref{app:sec:moving_goal}. Because CEM cannot meet the prescribed real-time budget, we use the more efficient iCEM controller~\citep{icem} for the neural-dynamics baselines, while Bilinear World Models retain the same planner used in the remaining experiments. As shown in Figure~\ref{fig:real-time}, our method achieves the highest success rate in all four environments, with 68\%, 77\%, 71\%, and 93\% success, compared with 46\%, 47\%, 45\%, and 63\% for the strongest baseline. Averaged across tasks, this corresponds to an improvement of approximately 27 percentage points.

\begin{figure}[b]
    \centering
    \begin{subfigure}{\textwidth}
        \centering
        \includegraphics[width=0.9\linewidth]{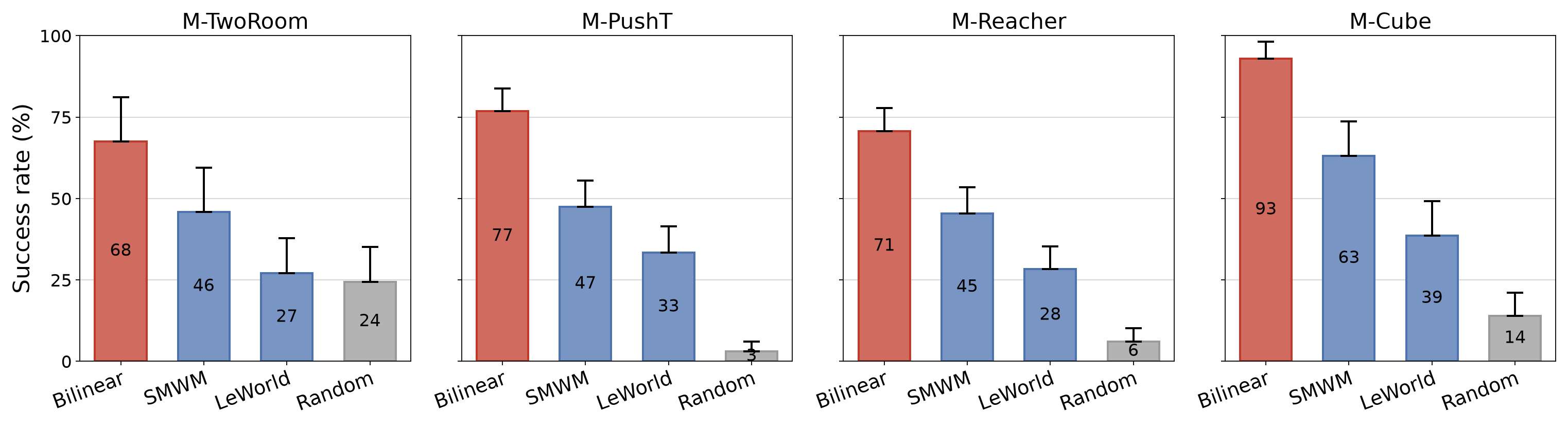}
    \end{subfigure}
    \caption{\textbf{Control performance on moving-goal tasks.}
    Success rate on M-TwoRoom, M-PushT, M-Reacher, and M-Cube under a 20\,Hz control loop. The target follows a random walk, requiring repeated replanning under controller latency. SMWM and LeWorld use iCEM~\citep{icem}, while Bilinear World Models use the same Gauss--Newton controller as in the other experiments. The Random baseline applies uniformly sampled actions.}
    \label{fig:real-time}
\end{figure}

%% file: statements.tex
\clearpage
\subsection*{AI use statement}

In this work, we used generative AI tools to assist with the implementation of experimental methods and the writing and presentation of proofs. We additionally used generative AI tools to suggest experimental parameters, create and edit software code, identify relevant literature, and edit portions of the manuscript to improve clarity and readability. All AI-assisted content was reviewed by the authors. In particular, AI-generated or AI-modified code was manually inspected and tested for correctness, and all AI-assisted mathematical arguments were independently verified by the authors. The authors take full responsibility for the final content of the manuscript, including any text, claims, code, or other artifacts produced with the assistance of generative AI.

We did not use generative AI tools to develop theoretical models or conceptual frameworks, formulate mathematical claims, provide critical ingredients for proving mathematical claims, propose or refine research hypotheses, design or provide feedback on the research methodology or experiments, generate synthetic datasets, assist with translation, clean or reformat datasets, perform qualitative or thematic data analysis, or interpret experimental results.

\subsection*{Ethics statement}

Given the nature of the work presented in this manuscript, we do not identify any additional ethical considerations that require discussion.

\subsection*{Reproducibility statement}

To facilitate reproducibility, we provide a ready-to-run script for reproducing the main experiments reported in the manuscript. The script launches a Docker container to provide a controlled and isolated computational environment. It has been designed to be simple and self-contained, with comments included to improve readability and facilitate its use. A detailed description of the hardware and software environment used to obtain the experimental results is provided in the Experimental section.

For the theoretical results, complete proofs of all formal statements, together with additional explanations where appropriate, are provided in the Appendix.

%% file: 06_appendix/input.tex
\section*{Appendix}

\input{06_appendix/a_proofs}

\input{06_appendix/b_experiment_detail}
\input{06_appendix/c_extra_results}

%% file: 06_appendix/a_proofs.tex
\section{Proof of Theorem \ref{thm:universal_koopman_representation}}
\label{apx:proof_bilinear_thm}
\begin{proof}
By Assumption~\ref{ass:action_affine}, the dynamics are control-affine. A standard result for Koopman representations of control-affine systems establishes that there exists a possibly infinite-dimensional collection of state-dependent observables whose evolution is bilinear in the observables and the action~\citep[Corollary II.1]{koopman-adv}. We choose this collection to include the original state coordinates and denote the resulting embedding by $z=E(s)$. The lifted dynamics can then be written as
\begin{align}
\dot z
=
Az
+
\left(B+Cz\right)a
=
Az+M(z)a.
\end{align}

Since the state coordinates are included among the observables, $E$ is injective and there exists a linear decoding operator $D$ such that
\begin{align}
DE(s)=s.
\end{align}

Applying the recovery operator $D$ to the lifted dynamics and using $Dz=DE(s) = s$ gives
\begin{align}
\dot s
=
DAz+DM(z)a.
\end{align}
Since this realization reproduces the original dynamics, comparison with Assumption~\ref{ass:action_affine} yields
\begin{align}
DM(E(s))
=
\begin{bmatrix}
f_1(s) & \cdots & f_m(s)
\end{bmatrix}.
\end{align}
Now suppose that, for some $v\in\mathbb R^m$,
\begin{align}
M(E(s))v=0.
\end{align}
Applying $D$ to both sides gives
\begin{align}
\begin{bmatrix}
f_1(s) & \cdots & f_m(s)
\end{bmatrix}v
=
0.
\end{align}
By Assumption~\ref{ass:independent_actions}, the matrix on the left has rank $m$, and hence $v=0$. Therefore,
\begin{align}
\ker M(E(s))=\{0\},
\end{align}
so $M(E(s))$ has full column rank for every $s\in\mathcal S$. Consequently, its Gram operator is positive definite,
\begin{align}
M(E(s))^\top M(E(s))\succ0,
\qquad
\forall s\in\mathcal S,
\end{align}
which completes the proof.
\end{proof}

%% file: 06_appendix/b_experiment_detail.tex
\section{Experimental Details}

This section presents extended experimental evaluations of four planning tasks introduced in \cite{leworld}. These tasks are intended to be solved by an agent using only observations of the world. In TwoRooms (2D), an agent is placed in one of two spaces divided by a wall and connected through a small opening and required to reach a desired target position, which may involve crossing the opening. This task is intended to model physical constraints of the environment, since the walls cannot be crossed by the agent. The PushT (2D) problem requires an agent to make contact with a T-shaped object and rotate, i.e., push it, to reach a desired position. This task is intended to model situations that are affected by external movable objects that require interaction. In Reacher (2D), a fixed arm with one joint is allowed to move within a circle to reach a target object, thus modeling rotations and multiple degrees of freedom. In Cube (3D), a gripper arm is required to catch a cube and move it to a target location, modeling complex actuation in spatial conditions.

\textbf{Hardware.} All experiments in the manuscript were performed on a single workstation with one NVIDIA GeForce RTX 3090 (24\,GB, driver 550.144.03), an AMD Ryzen Threadripper 3960X (24 cores, 48 threads) and 125\,GB of RAM. Every world model was trained on that single GPU in \texttt{bf16}. The software stack was PyTorch 2.7.1 (CUDA 12.6,
cuDNN 9.5.1), Lightning 2.6.6, MuJoCo 3.13.0, \texttt{dm\_control} 1.0.46 and Gymnasium 1.3.0.

\subsection{Solving the constrained optimization problem~\ref{eq:problem}}
\label{app:sec:problem}

The problem~\ref{eq:problem} is formulated using the normalized action map
$M(z)=Q(z)R$, where the state-dependent factor $Q(z)$ is obtained by applying a Cholesky--QR decomposition to the bilinear parameterization
$ M(z)=B+Cz$. This normalization changes the exact functional form of the dynamics, which are therefore no longer strictly bilinear. However, it preserves the state-dependent action subspace generated by the original bilinear map. In particular, whenever $ M(z)$ has full column rank and $R$ is nonsingular,
\begin{equation}
    \operatorname{col}\!\left(M(z)\right)
    =
    \operatorname{col}\!\left(Q(z)\right)
    =
    \operatorname{col}\!\left( M(z)\right).
    \label{app:eq:column_space}
\end{equation}
Thus, the bilinear parameterization determines the state-dependent action directions, while the normalization provides a well-conditioned basis for the same subspace.

Specifically, the normalized action map is given by
\begin{equation}
    M(z)=Q(z)R,
    \qquad
    Q(z)=\mathrm{CholQR}\!\left(B+Cz\right),
    \label{app:eq:qr_parameterization}
\end{equation}
where $Q(z)\in\mathbb{R}^{d\times m}$ has orthonormal columns and $R\in\mathbb{R}^{m\times m}$ is a learned state-independent matrix. Consequently,
\begin{equation}
\begin{aligned}
    M(z)^\top M(z)
    &=R^\top Q(z)^\top Q(z)R \\
    &=R^\top R.
\end{aligned}
\label{app:eq:gram_factorization}
\end{equation}
Hence, the full-column-rank condition of the normalized action map is independent of the latent state. In particular,
\begin{equation}
    M(z)^\top M(z)\succ0
    \quad\Longleftrightarrow\quad
    \sigma_{\min}(R)>0.
    \label{app:eq:singular_value_condition}
\end{equation}
The matrix $Q(z)$ therefore determines the orientation of the state-dependent action directions, while their scaling and conditioning are controlled by the constant matrix $R$.

Although action recoverability can be expressed through the hard constraint $R^\top R\succeq\varepsilon I$, enforcing this constraint exactly during stochastic optimization is inconvenient in practice. We therefore solve a regularized surrogate in which action recoverability is encouraged directly from the observed transitions,

\begin{equation}
\begin{aligned}
\underset{\phi_\theta,A,B,C,R}{\mathrm{minimize}}\quad
&\frac{1}{N}\sum_t
\Big\|
\phi_\theta(s_{t+1})
-
A\phi_\theta(s_t)
-
\mathrm{CholQR}\!\left(B+C\phi_\theta(s_t)\right)R\,a_t
\Big\|^2 \\
&\quad+
\lambda\frac{1}{N}\sum_t
\Bigg\|
a_t
-
\Big[
\mathrm{CholQR}\!\left(B+C\phi_\theta(s_t)\right)R
\Big]^\dagger
\Big(
\phi_\theta(s_{t+1})
-
A\phi_\theta(s_t)
\Big)
\Bigg\|^2.
\end{aligned}
\label{app:eq:optimization}
\tag{$\hat{\mathrm P}$}
\end{equation}

The first term fits the normalized latent dynamics to the observed transitions. The second directly encourages the action that generated each transition to be recoverable from the corresponding latent displacement. This discourages degenerate solutions in which distinct actions produce indistinguishable latent transitions and provides a practical surrogate for the full-rank requirement in Problem~\eqref{eq:problem}. Although this surrogate does not explicitly enforce $R^\top R\succeq\varepsilon I$, we empirically verify in Appendix~\ref{app:sec:training_metrics} that $\sigma_{\min}(R)$ remains bounded away from zero throughout training across all environments, and hence that the learned action map remains full column rank. Since the Cholesky--QR operation is differentiable, gradients can be propagated through the normalization and the complete model can be trained end-to-end using standard gradient-based optimization. For all Bilinear World Model experiments, we use $\lambda=30$.

\subsection{Controllers}
\label{app:sec:controller}

We provide implementation details for the controllers used in the experiments of the main text. Specifically, we describe the Cross-Entropy Method (CEM) used for the standard planning experiments, the iterative CEM (iCEM) variant used for real-time control, and our implementation of the Gauss--Newton (GN).

\subsubsection{Cross-Entropy Method}

The Cross-Entropy Method (CEM)~\citep{cem} is the sampling-based planner used in the original evaluations of LeWM and Sensorimotor. Given a planning horizon $H$, CEM maintains a Gaussian distribution over action sequences
\begin{align}
    \mathbf{a}
    =
    (a_0,\ldots,a_{H-1})
    \sim
    \mathcal{N}(\mu,\Sigma).
\end{align}
At each iteration, $N$ action sequences are sampled from this distribution and rolled out through the learned world model. Each sequence is evaluated according to the planning objective
\begin{align}
    \mathcal{L}(\mathbf{a})
    =
    \frac{1}{2}
    \left\|
    z_H(\mathbf{a})-z^\star
    \right\|^2.
\end{align}
The $K$ sequences with lowest cost are retained as the elite set $\mathcal{E}$. The sampling distribution is then updated according to the empirical mean and variance of these elites,
\begin{align}
    \mu
    &\gets
    \frac{1}{K}
    \sum_{\mathbf{a}\in\mathcal{E}}
    \mathbf{a},
    \\
    \Sigma
    &\gets
    \frac{1}{K}
    \sum_{\mathbf{a}\in\mathcal{E}}
    (\mathbf{a}-\mu)(\mathbf{a}-\mu)^\top.
\end{align}
This procedure is repeated for a fixed number of iterations, and the lowest-cost sequence found during the optimization is selected for execution.

CEM treats the learned world model as a black box and does not use the Jacobian of the predicted endpoint with respect to the actions. Its computational cost is therefore dominated by the large number of candidate trajectories that must be rolled out and evaluated during the search.

\subsubsection{Improved Cross-Entropy Method}

We additionally evaluate iCEM, a more sample-efficient variant of CEM. The overall optimization procedure remains unchanged, but the sampling strategy is modified to concentrate computation on more promising and temporally coherent action sequences.

In particular, instead of sampling independent white Gaussian perturbations, our implementation uses temporally correlated colored noise. Starting from white noise $\xi$, its Fourier coefficients are scaled according to
\begin{align}
    \widehat{\xi}_k^{\,\mathrm{col}}
    \propto
    f_k^{-\beta/2}
    \widehat{\xi}_k,
\end{align}
followed by normalization back to the desired variance. We use $\beta=2$, which places greater weight on low-frequency components and therefore produces smoother action sequences across the planning horizon. Setting $\beta=0$ recovers independent white-noise sampling as in standard CEM.

The implementation also reuses elite trajectories between consecutive optimization iterations. Specifically, the best five sequences from the previous iteration are inserted directly into the next candidate population, together with the current distribution mean. The remaining candidates are newly sampled from the updated distribution. This allows promising trajectories to persist across iterations without requiring them to be rediscovered by sampling.

Our implementation does not use population decay, and the number of sampled trajectories therefore remains fixed across optimization iterations. It also does not shift the optimized action sequence between successive replanning steps. Unless an initial action sequence is explicitly provided, each planning problem is initialized from a zero-mean action distribution.

\subsubsection{Gauss--Newton Planner}
\label{app:sec:gn-controller}

We implement the planner of Section~\ref{sec:control} in a receding-horizon fashion. At each replan, the current observation and goal are encoded into $z_0$ and $z^\star$, and an action sequence $\mathbf{a}=(a_0,\ldots,a_{H-1})$ is optimized according to~\eqref{eq:planning-objective}. Each planning step corresponds to the same action block used during training, and the optimized sequence is executed before replanning from the resulting observation.

\paragraph{Reduced action space.}
In practice, each model action contains several low-level environment actions. Optimizing all of these coordinates independently produces poorly conditioned Gauss--Newton systems, since variations within an action block have little distinguishable effect on the learned dynamics. We therefore restrict the update to a lower-dimensional coherent subspace in which the low-level actions within each block vary together. This reduces the number of optimized variables from $50$ to $10$ for the 2D tasks with $H=5$, and from $125$ to at most $15$ for Cube. The corresponding lifting from the reduced variables to the complete action sequence is denoted by $E$.

\paragraph{Endpoint sensitivities.}
Although Section~\ref{sec:control} gives an explicit expression for the endpoint Jacobian, in the reported experiments we estimate its restriction to the coherent subspace using directional finite differences. At each iteration, the nominal action sequence and all perturbed sequences are evaluated jointly in a single batched rollout,
\begin{align}
\widehat{\mathcal J E}_{:,j}
=
\frac{
z_H(\mathbf{a}+\varepsilon E_{:,j})
-
z_H(\mathbf{a})
}{\varepsilon},
\qquad
\varepsilon=10^{-3}.
\end{align}
This directly captures the sensitivity of the complete nonlinear predictor, including the state dependence of the normalized action directions and the learned residual dynamics, without explicitly differentiating through the full rollout.

\paragraph{Damped update and line search.}
Let $C=\widehat{\mathcal J E}$ and $e=z_H(\mathbf{a})-z^\star$. The reduced Gauss--Newton direction is obtained from
\begin{align}
\left[
C^\top C+(c+\mu)I
\right]\delta_c
=
-
\left[
C^\top e+cE^\top\mathbf{a}
\right],
\qquad
\delta=E\delta_c,
\end{align}
with damping $\mu=10^{-4}$. We then evaluate the step sizes
$\{1,\frac12,\frac14,\frac18\}$ and accept the first one that decreases the true nonlinear planning objective. All candidate step sizes are evaluated jointly in a single batched rollout.

\paragraph{Initialization.}
Since Gauss--Newton is a local method, we optionally initialize the action sequence from a small set of candidate initializations containing the current warm start, several gradient-aligned sequences, and random perturbations. All candidates are evaluated in parallel and the lowest-cost sequence is used to initialize the optimization. The number of candidates, planning horizon, reduced dimension, and iteration budget used for each task are reported in Table~\ref{tab:a}.

\paragraph{Planning residual.}
The structured model is intentionally restricted during representation learning, but this restriction can leave systematic prediction errors in contact-rich environments. We therefore train an auxiliary residual model \emph{a posteriori}, after the encoder and structured dynamics have been learned and fixed. Consequently, the residual does not influence the learned representation or the parameters $A,B,C,R$, and is used only to improve the transition model employed during planning.

During planning, we augment the structured transition
$S(z,a)=Az+M(z)a$ as
\begin{align}
\label{eq:planning-residual}
f_{\mathrm{plan}}(z,a)
=
S(z,a)
+
\bigl(I-P(z)P(z)^\top\bigr)r_\psi(z,a),
\end{align}
where $P(z)$ spans the coherent control directions introduced above. The projection prevents the residual from modifying these directions while allowing it to capture components of the transition that cannot be represented by the structured model alone.

The endpoint sensitivities used by the controller are computed through this complete planning model. Therefore, the directional finite differences described below automatically account for both the structured dynamics and the residual correction.

\begin{table}[]
\centering\small
\begin{tabular}{lcccccccc}
\toprule
Task & $d_a$ & $H$ & $K$ & $n_b$ & $K_c$ & iters & $n_{\text{cand}}$ & $\rho$ \\
\midrule
PushT   & 2 & 5 & 50  & 5 & 10 & 40 & 16 & 0.5 \\
Reacher & 2 & 3 & 30  & 3 & 6  & 20 & 64 & 2.0 \\
TwoRoom & 2 & 5 & 50  & 3 & 6  & 40 & 0  & --- \\
Cube    & 5 & 5 & 125 & 3 & 15 & 20 & 0  & --- \\
\bottomrule
\end{tabular}
\caption{Gauss--Newton planner settings for the reported cells. $H$: planning
horizon in action blocks; $m=k\,d_a$ with block size $k=5$; $K=Hm$ raw unknowns,
$K_c$ coherent unknowns; $n_b$: time blocks refreshed per iteration.
Shared across all tasks: $\lambda_g=1$, $c=0.01$, $\mu=10^{-4}$,
$\varepsilon=10^{-3}$, line search $\{1,\tfrac12,\tfrac14,\tfrac18\}$,
$n_{\text{starts}}=1$, tolerance $10^{-6}$.}
\label{tab:a}
\end{table}

\subsection{Moving goal implementation}
\label{app:sec:moving_goal}

\paragraph{Task construction.}
We construct moving-goal variants of the four benchmark environments by allowing the target state to follow a random walk throughout the episode. At each environment step, the goal state $g_t$ is displaced by
$\sigma v_{\max}u$, where $u$ is a uniformly sampled direction, $v_{\max}$ is the characteristic maximum per-step motion of the corresponding system, and $\sigma$ controls the relative goal speed. We estimate $v_{\max}$ as the 95th percentile of per-step motion in the expert data, obtaining $22.4$\,px for PushT, $0.169$\,rad for Reacher, and $0.028$\,m for Cube. Invalid displacements that leave the workspace or cross obstacles are resampled. Goal trajectories are seeded per task so that all methods are evaluated on identical moving targets. The resulting environments are denoted M-TwoRoom, M-PushT, M-Reacher, and M-Cube. The goal is re-rendered at every step, and the original environment-specific success criterion is used.

\paragraph{Real-time evaluation protocol.}
Planning latency is explicitly incorporated into the control loop. For each method and controller, we first measure the wall-clock latency $\tau$ of a single planning call on one environment. Given a control period $\Delta$, this corresponds to a delay of
$L=\lceil \tau/\Delta\rceil$ environment steps. A plan computed from the observation and goal at step $t$ therefore becomes available at step $t+L$. During this interval, the agent continues executing the remainder of its previous plan, or applies a zero action if no action remains. Methods replan when the current plan is exhausted.

All experiments in Figure~\ref{fig:real-time} use a 20\,Hz control loop, corresponding to $\Delta=50$\,ms. The baselines use iCEM~\citep{icem}, since the execution time of standard CEM is incompatible with the real-time budget. Bilinear World Models use the same Gauss--Newton controller employed in the remaining experiments. For each baseline, we report the best-performing iCEM sample budget under this protocol.

\paragraph{Evaluation settings.}
Goal speeds are selected independently for each environment and expressed relative to the corresponding $v_{\max}$: $3.0$ for TwoRoom, $0.1$ for PushT, $0.05$ for Reacher, and $0.25$ for Cube. To avoid inflating success rates with tasks that require no meaningful control, we evaluate only episodes that are not solved by a policy that never acts. 

%% file: 06_appendix/c_extra_results.tex
\section{Additional Results}
\label{app:sec:experients}

This section provides additional empirical results complementing the main experiments. We first examine the optimization behavior of Bilinear World Models and verify that the learned action map remains nondegenerate throughout training. We then compare the performance of CEM, iCEM, and Gauss--Newton (GN) across the different world-model architectures. Finally, we provide additional visualizations of the geometry of the learned latent representations.

\subsection{Training dynamics and action-map conditioning}
\label{app:sec:training_metrics}

Problem~\ref{eq:problem} requires the state-dependent action map $M(z)$ to have full column rank. As described in Appendix~\ref{app:sec:problem}, our implementation parameterizes the action map as
\begin{equation}
    M(z)=Q(z)R,
    \label{app:eq:results_factorization}
\end{equation}
where $Q(z)$ has orthonormal columns and $R$ is a learned state-independent matrix. Therefore,
\begin{equation}
    M(z)^\top M(z)=R^\top R,
    \label{app:eq:results_gram}
\end{equation}
and the singular values of $M(z)$ are identical to those of $R$ for every latent state $z$. In particular,
\begin{equation}
    \sigma_{\min}\!\left(M(z)\right)
    =
    \sigma_{\min}(R),
    \qquad \forall z.
    \label{app:eq:results_singular_values}
\end{equation}
Consequently, no averaging or minimization over latent states is required to evaluate the conditioning of the learned action map.

Figure~\ref{app:fig:train} reports the latent forward-prediction loss together with $\sigma_{\min}(R)$ throughout training. Across all four environments, the prediction loss decreases steadily while the minimum singular value remains bounded away from zero. Thus, the learned action map remains full column rank throughout training while simultaneously fitting the observed latent dynamics.
\begin{figure}[]
    \centering

    \begin{subfigure}{0.49\textwidth}
        \centering
        \includegraphics[width=\linewidth]{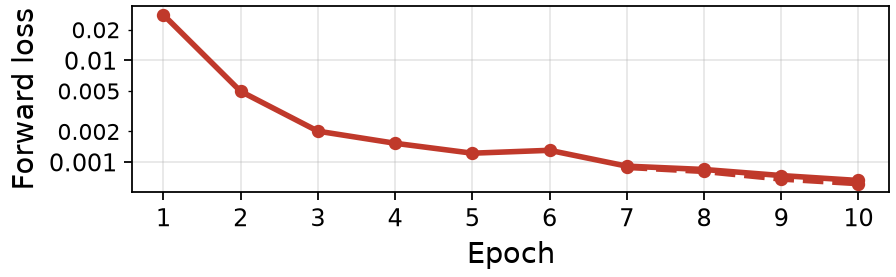}\\[0.3em]
        \includegraphics[width=\linewidth]{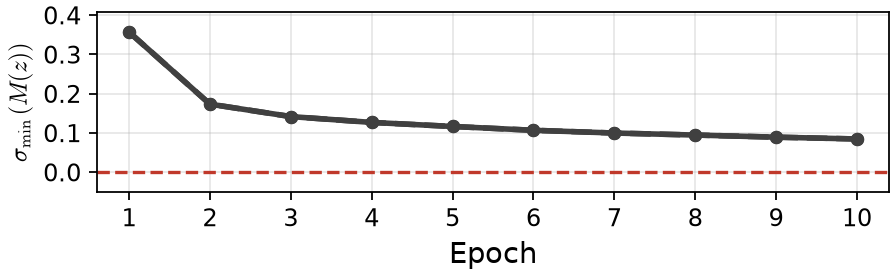}
        \caption{TwoRoom}
    \end{subfigure}
    \hfill
    \begin{subfigure}{0.49\textwidth}
        \centering
        \includegraphics[width=\linewidth]{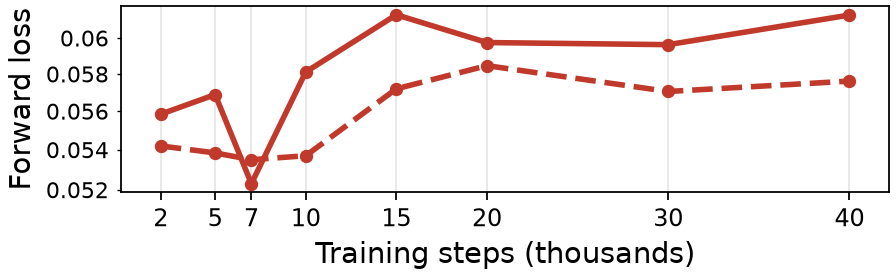}\\[0.3em]
        \includegraphics[width=\linewidth]{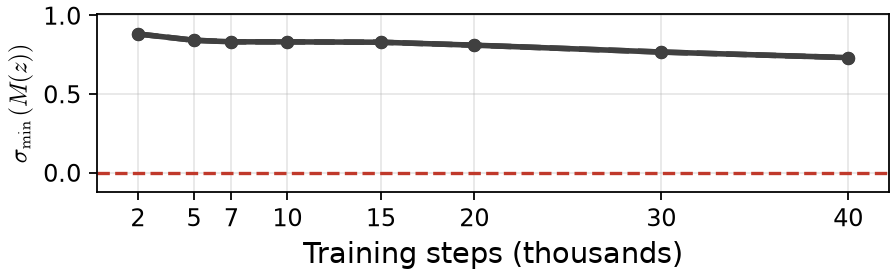}
        \caption{PushT}
    \end{subfigure}

    \vspace{0.5em}

    \begin{subfigure}{0.49\textwidth}
        \centering
        \includegraphics[width=\linewidth]{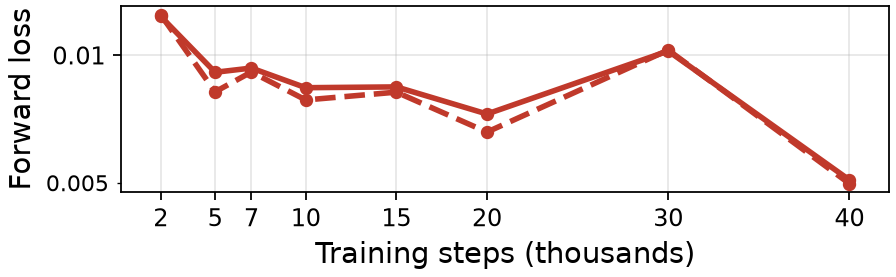}\\[0.3em]
        \includegraphics[width=\linewidth]{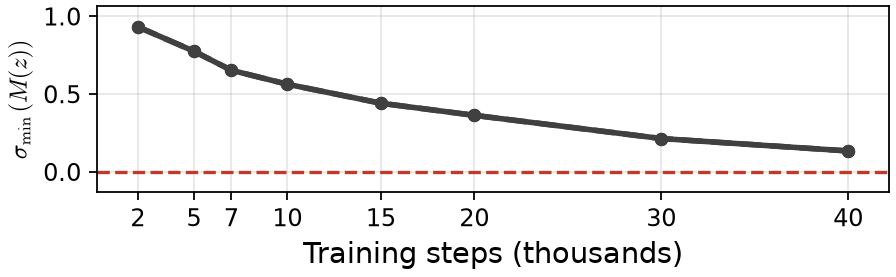}
        \caption{Reacher}
    \end{subfigure}
    \hfill
    \begin{subfigure}{0.49\textwidth}
        \centering
        \includegraphics[width=\linewidth]{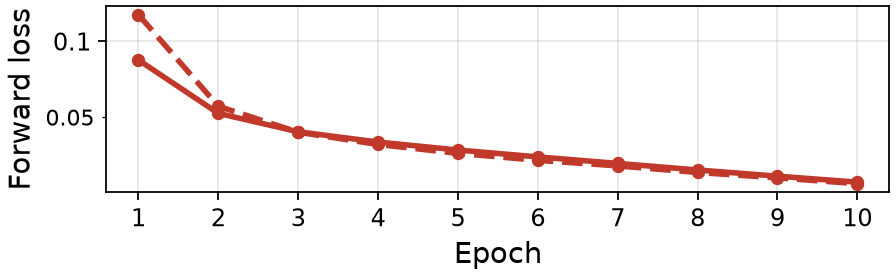}\\[0.3em]
        \includegraphics[width=\linewidth]{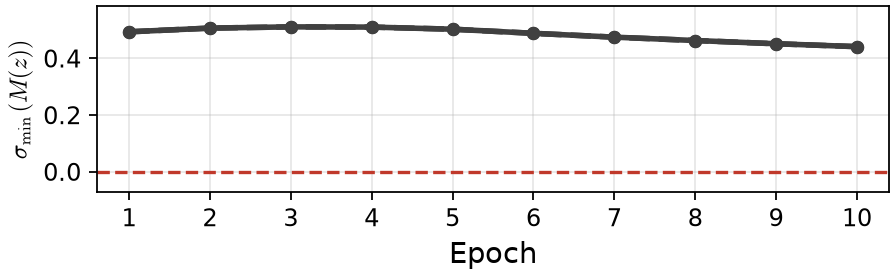}
        \caption{Cube}
    \end{subfigure}

    \caption{\textbf{Training dynamics of Bilinear World Models.} The latent forward-prediction loss decreases throughout training, while the minimum singular value $\sigma_{\min}(R)$ remains bounded away from zero. Since $M(z)=Q(z)R$ with $Q(z)$ orthonormal, $\sigma_{\min}(R)=\sigma_{\min}(M(z))$ for every latent state $z$. The learned action map therefore remains full column rank throughout training.}
    \label{app:fig:train}
\end{figure}

\subsection{Controller comparison across world models}
\label{app:sec:controller_comparison}

The main advantage of imposing structure on the latent dynamics is that this structure can be exploited directly during control. To separate the effect of the learned world model from that of the controller, we evaluate three planning methods---CEM, iCEM, and Gauss--Newton (GN)---on Bilinear World Models, SMWM~\citep{sensimotor}, and LeWorld~\citep{leworld}. Figure~\ref{app:fig:comparison} reports both success rate and planning time across the four environments.

GN benefits directly from the structured dynamics of our model. When paired with Bilinear World Models, it achieves strong control performance with substantially lower planning time. In contrast, its performance degrades when applied to SMWM and LeWorld, whose latent dynamics are represented by unconstrained nonlinear models and therefore do not provide the same structure to the optimizer. These results indicate that the computational advantage of GN is closely tied to the explicit form of the learned dynamics.

CEM provides a complementary comparison because it is agnostic to the structure of the world model and can therefore be applied uniformly to all three architectures. When used with Bilinear World Models, CEM achieves competitive success rates but incurs substantially greater planning cost than GN, since it does not exploit the available structure. The real-time variant iCEM reduces this computational cost, but its performance deteriorates on the more challenging environments, particularly Cube. For this reason, iCEM is primarily used in the real-time experiments of Section~\ref{sec:experiments}, where rapid replanning is required.

\begin{figure}[]
    \centering

    \begin{subfigure}{0.49\textwidth}
        \centering
        \begin{subfigure}{0.48\textwidth}
            \centering
            \includegraphics[width=\linewidth]{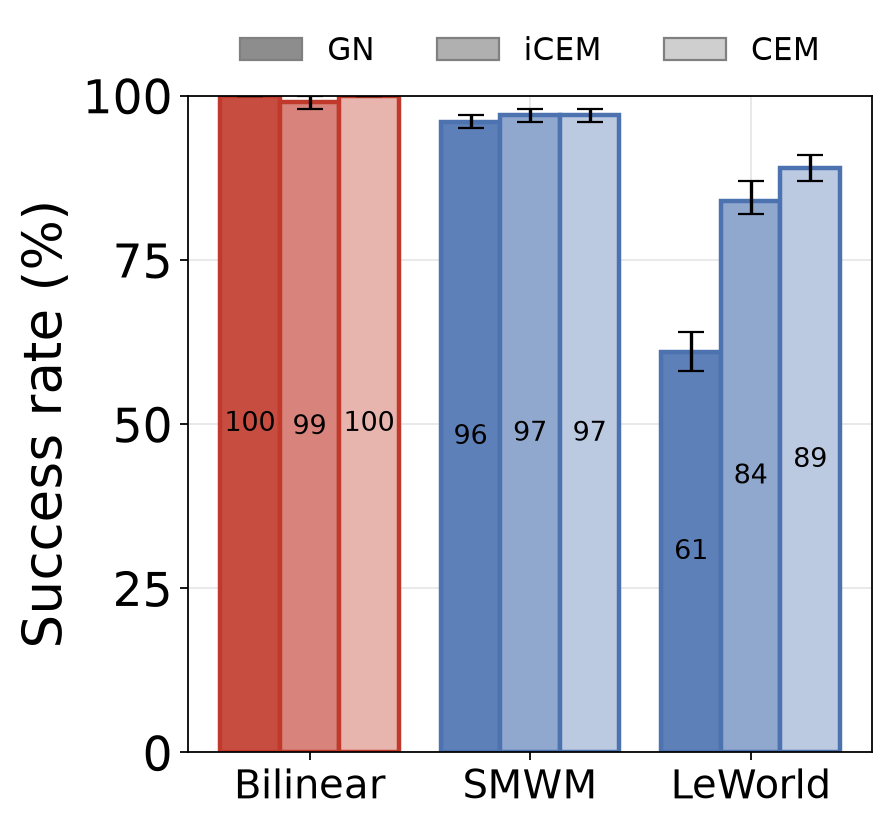}
        \end{subfigure}
        \hfill
        \begin{subfigure}{0.48\textwidth}
            \centering
            \includegraphics[width=\linewidth]{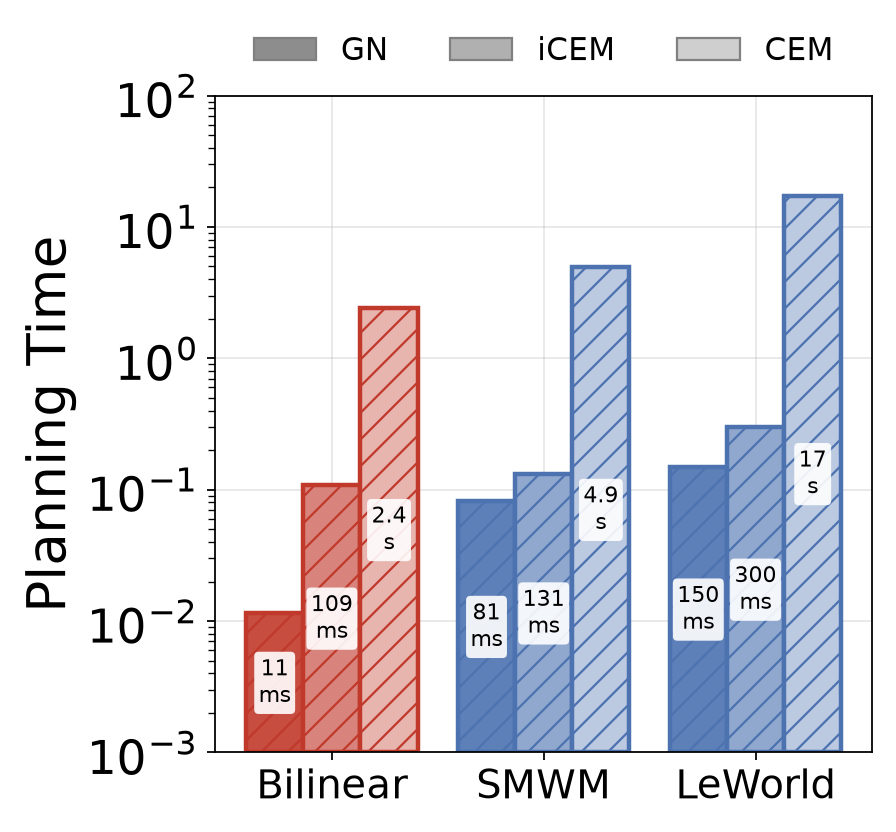}
        \end{subfigure}
        \caption{TwoRoom}
    \end{subfigure}
    \hfill
    \begin{subfigure}{0.49\textwidth}
        \centering
        \begin{subfigure}{0.48\textwidth}
            \centering
            \includegraphics[width=\linewidth]{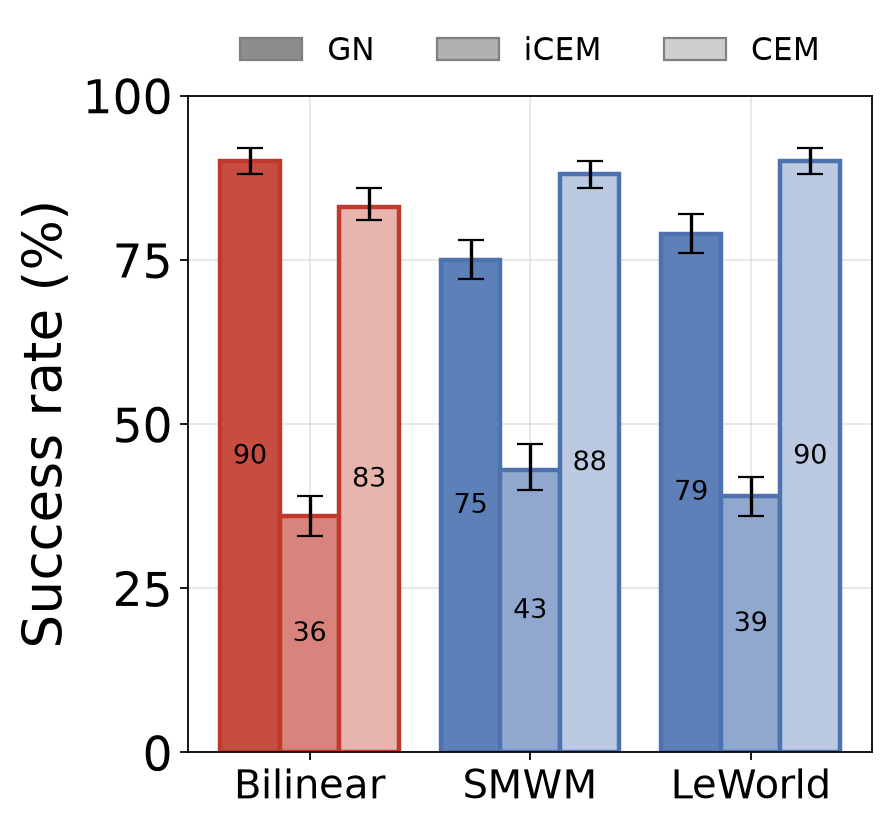}
        \end{subfigure}
        \hfill
        \begin{subfigure}{0.48\textwidth}
            \centering
            \includegraphics[width=\linewidth]{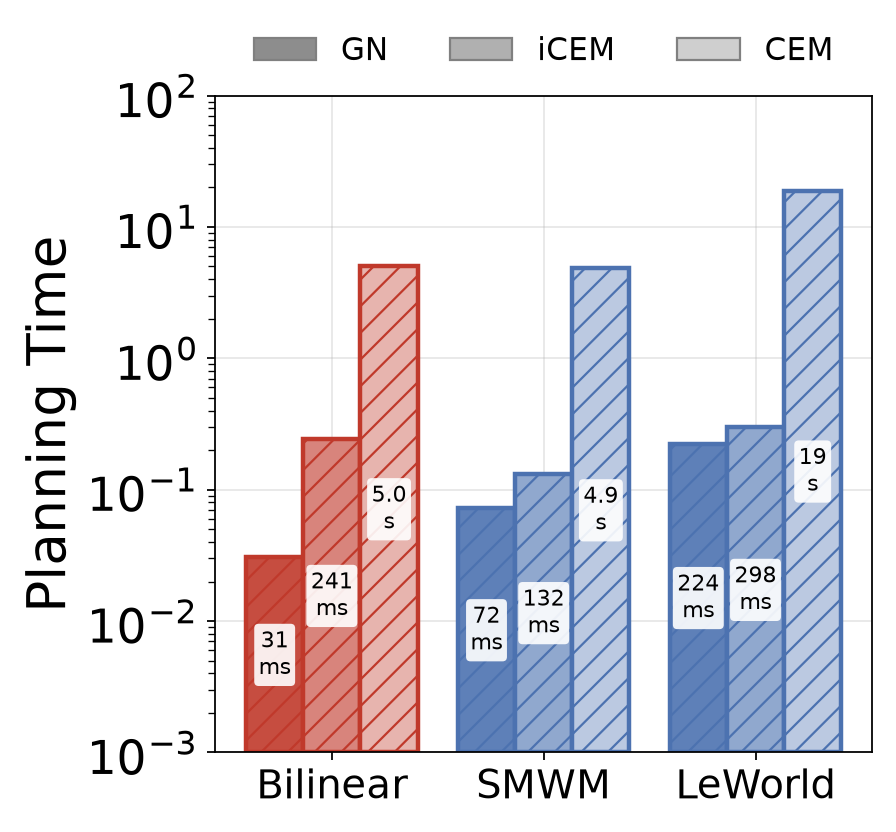}
        \end{subfigure}
        \caption{PushT}
    \end{subfigure}

    \vspace{0.8em}

    \begin{subfigure}{0.49\textwidth}
        \centering
        \begin{subfigure}{0.48\textwidth}
            \centering
            \includegraphics[width=\linewidth]{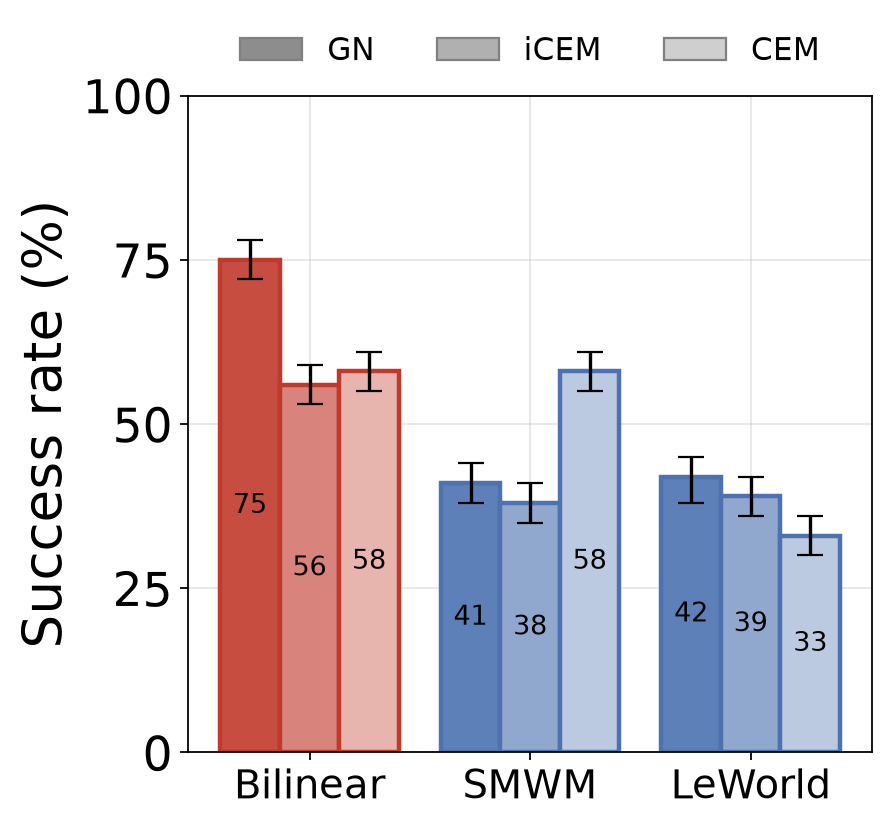}
        \end{subfigure}
        \hfill
        \begin{subfigure}{0.48\textwidth}
            \centering
            \includegraphics[width=\linewidth]{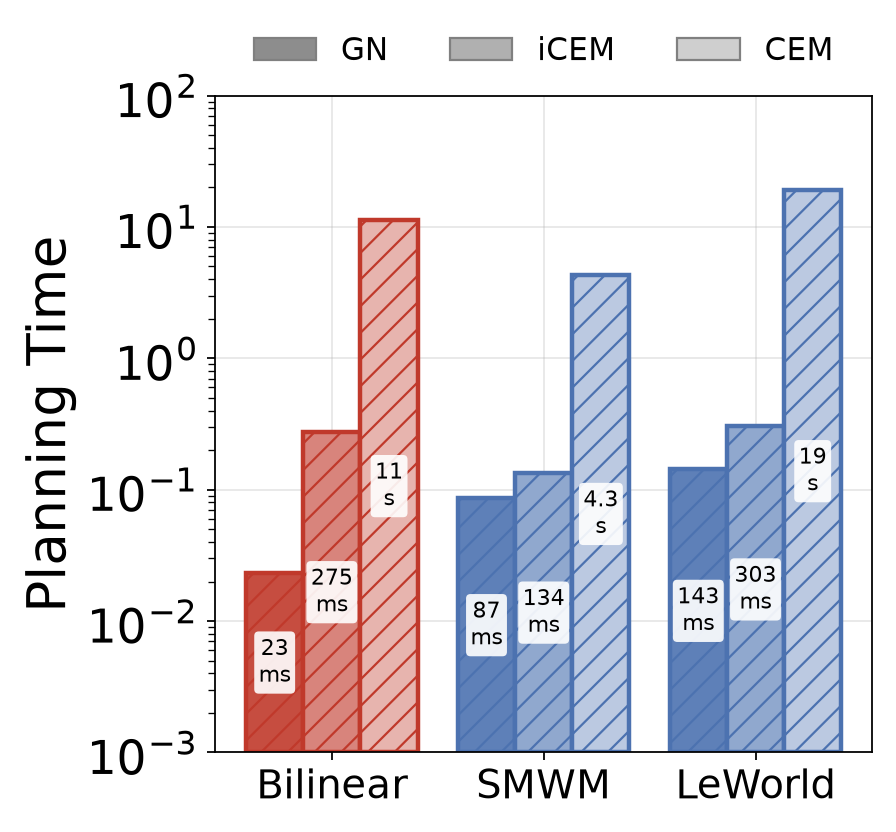}
        \end{subfigure}
        \caption{Reacher}
    \end{subfigure}
    \hfill
    \begin{subfigure}{0.49\textwidth}
        \centering
        \begin{subfigure}{0.48\textwidth}
            \centering
            \includegraphics[width=\linewidth]{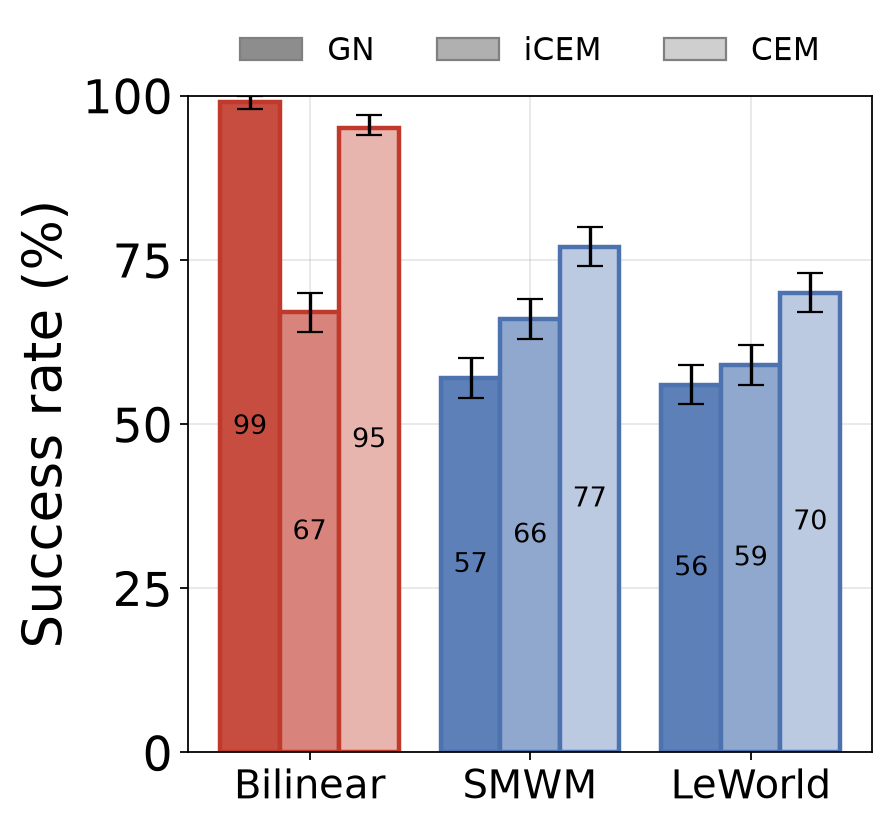}
        \end{subfigure}
        \hfill
        \begin{subfigure}{0.48\textwidth}
            \centering
            \includegraphics[width=\linewidth]{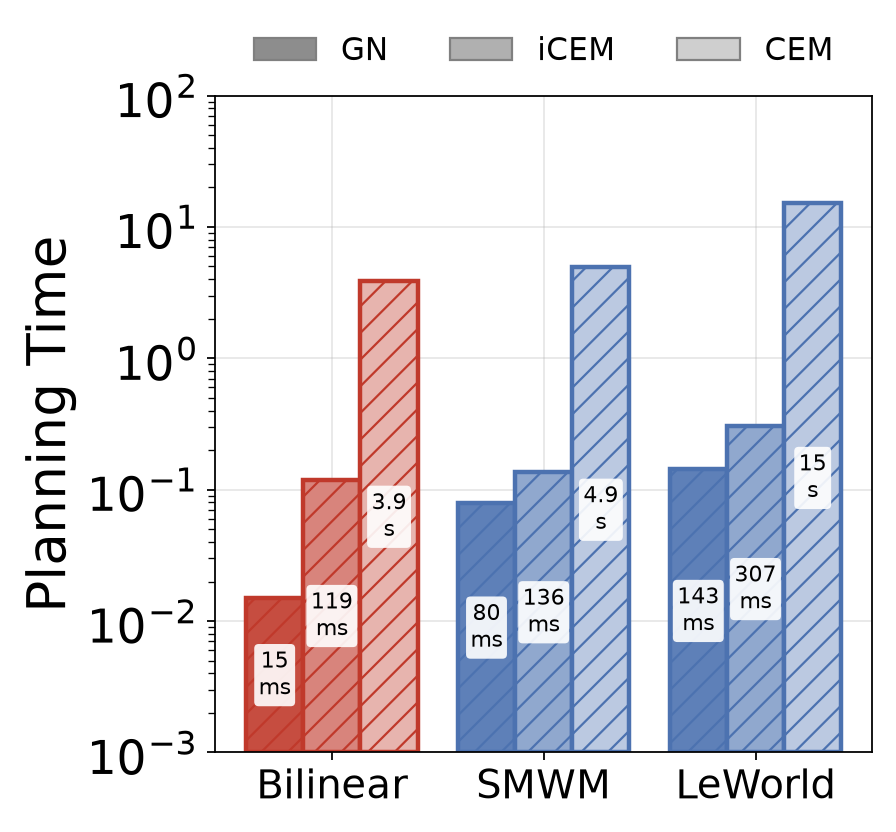}
        \end{subfigure}
        \caption{Cube}
    \end{subfigure}

    \caption{\textbf{Controller performance across world models.} Success rate and planning time are reported for CEM, iCEM, and Gauss--Newton (GN) applied to Bilinear World Models, SMWM, and LeWorld. GN is particularly effective when paired with Bilinear World Models, where it can exploit the explicit structure of the latent dynamics. CEM remains applicable across all architectures but incurs substantially higher planning cost, while iCEM reduces computation at the expense of degraded performance on the more challenging tasks.}
    \label{app:fig:comparison}
\end{figure}

\subsection{Geometry of the learned latent representations}
\label{app:sec:pca}

World models jointly learn representations of observations and models of their dynamics. Although the observations are high dimensional, the underlying physical systems evolve on substantially lower-dimensional state spaces. We therefore further examine the geometry of the learned latent representations using principal component analysis (PCA).

Figure~\ref{app:fig:surfaces} shows the latent states projected onto their first three principal components. Across the environments, the learned representations organize into smooth, low-dimensional structures rather than filling the ambient latent space. For TwoRoom, PushT, and Reacher, the dominant variation is largely captured by a two-dimensional surface, whereas Cube exhibits a richer geometry that is not fully represented by a three-dimensional projection.

Interestingly, the apparent dimensionality of these structures is closely related to the dimensionality of the corresponding control spaces. The simpler environments have two-dimensional action spaces and exhibit approximately two-dimensional latent geometry, while Cube has a higher-dimensional action space and correspondingly more complex latent structure. While PCA alone does not establish the exact intrinsic dimension of the learned representation, these results suggest that the latent space organizes around a compact set of degrees of freedom that are relevant for control.

\begin{figure}[p]
    \centering

    \begin{subfigure}{\linewidth}
        \centering
        \includegraphics[width=\linewidth]{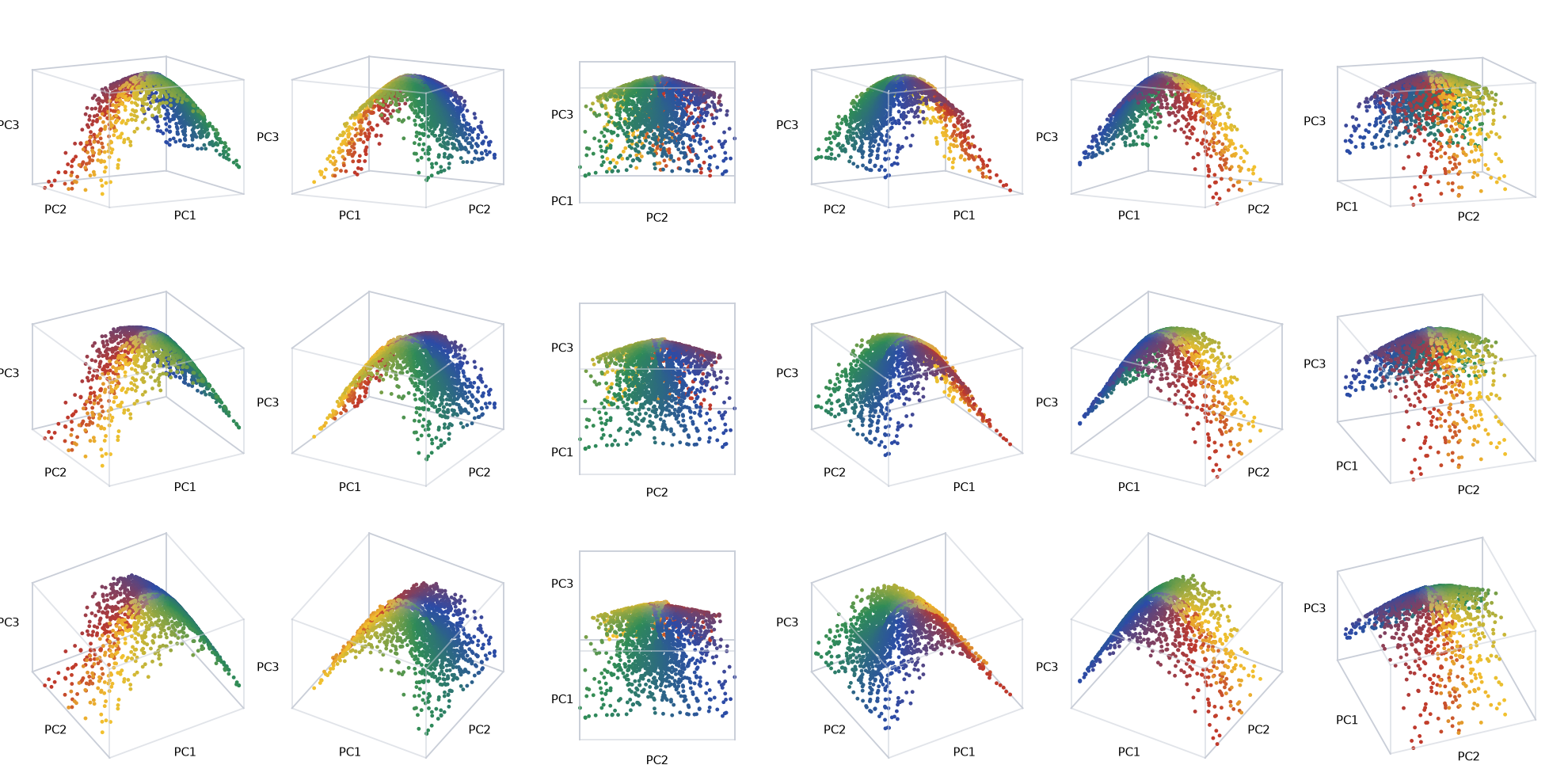}
        \caption{TwoRoom}
        \label{fig:pca3d_tworoom}
    \end{subfigure}

    \vspace{0.5em}

    \begin{subfigure}{\linewidth}
        \centering
        \includegraphics[width=\linewidth]{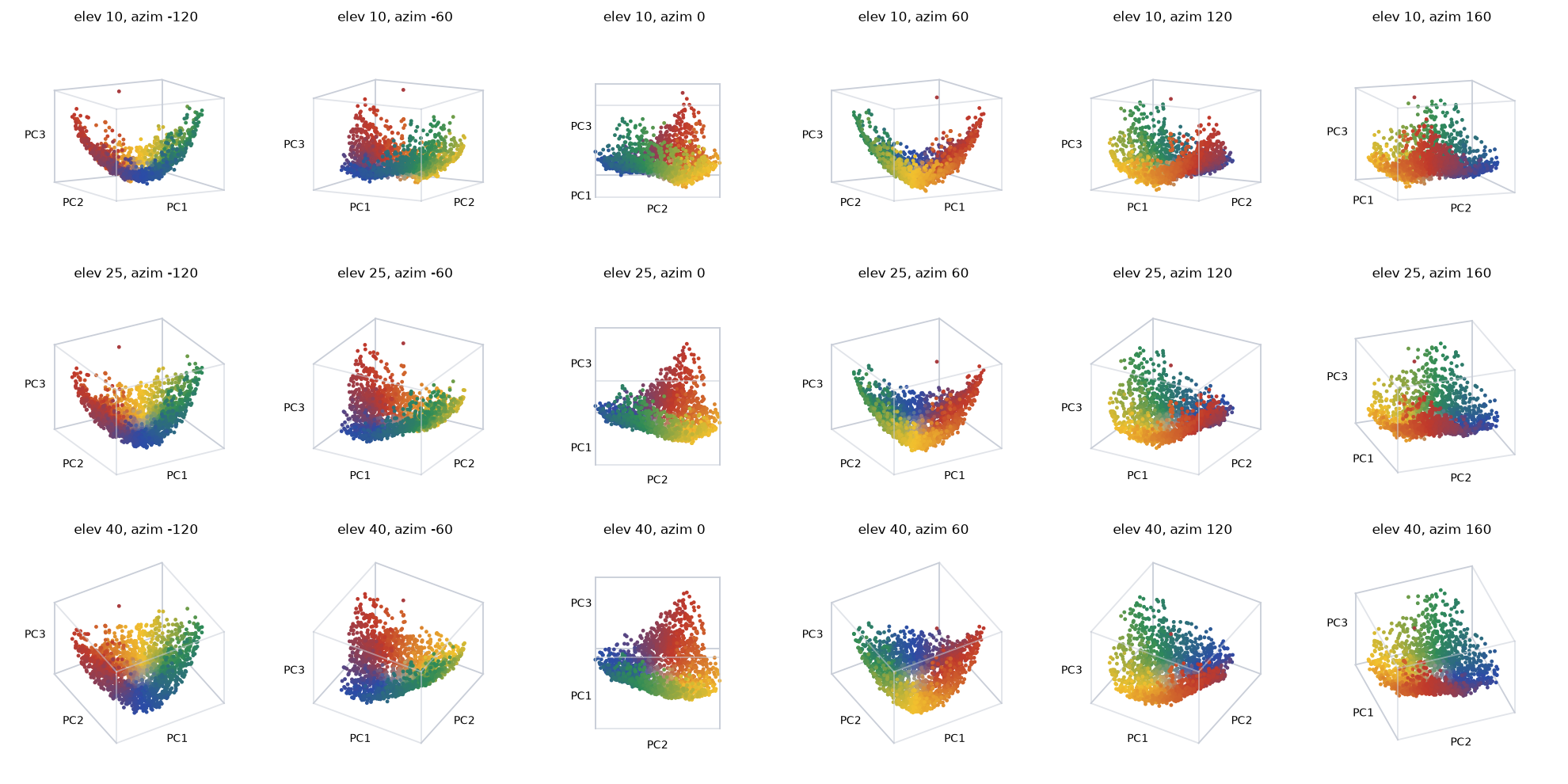}
        \caption{PushT}
        \label{fig:pca3d_pusht}
    \end{subfigure}
\end{figure}

\begin{figure}[p]
    \ContinuedFloat
    \centering

    \begin{subfigure}{\linewidth}
        \centering
        \includegraphics[width=\linewidth]{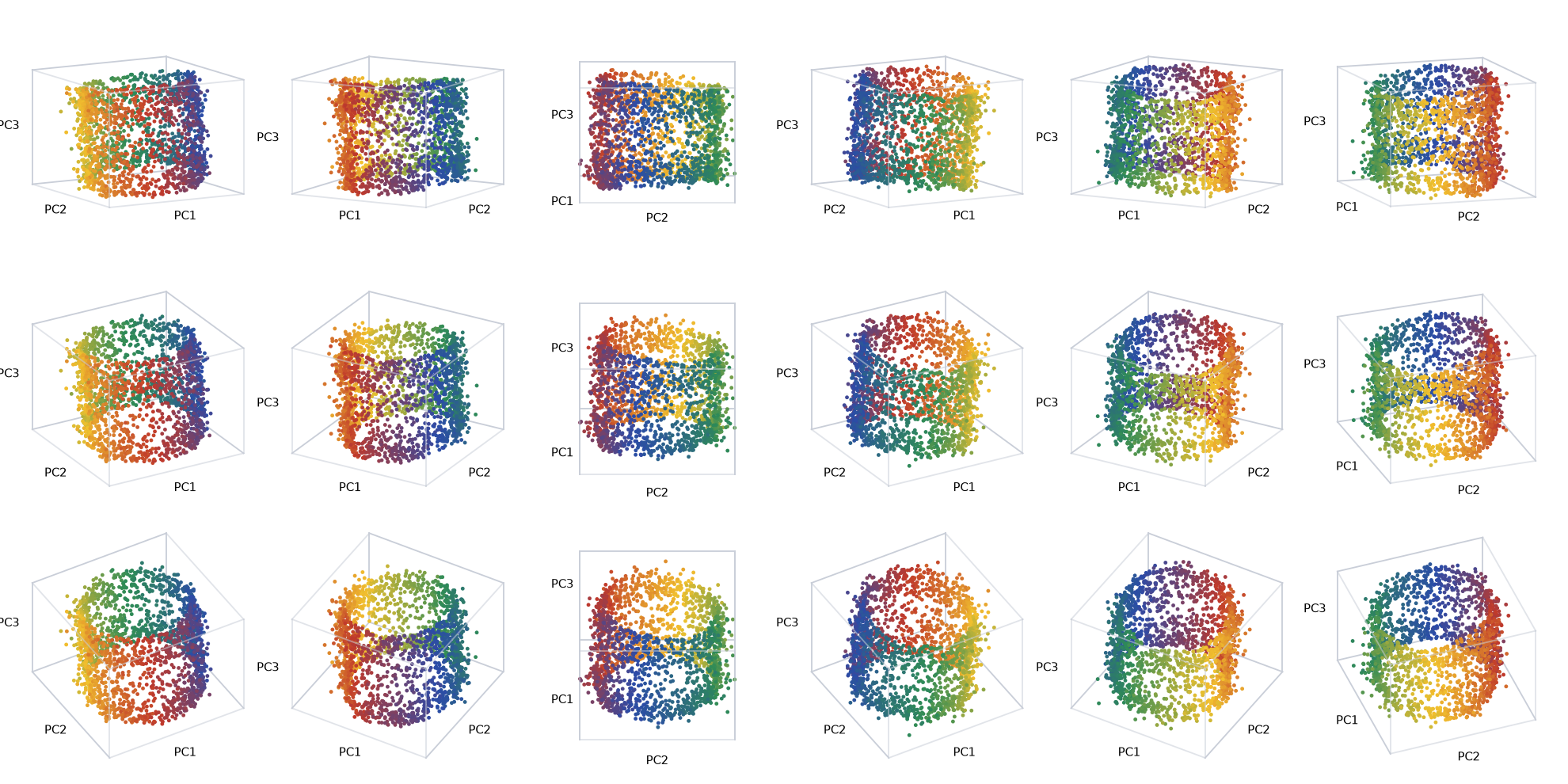}
        \caption{Reacher}
        \label{fig:pca3d_reacher}
    \end{subfigure}

    \vspace{0.5em}

    \begin{subfigure}{\linewidth}
        \centering
        \includegraphics[width=\linewidth]{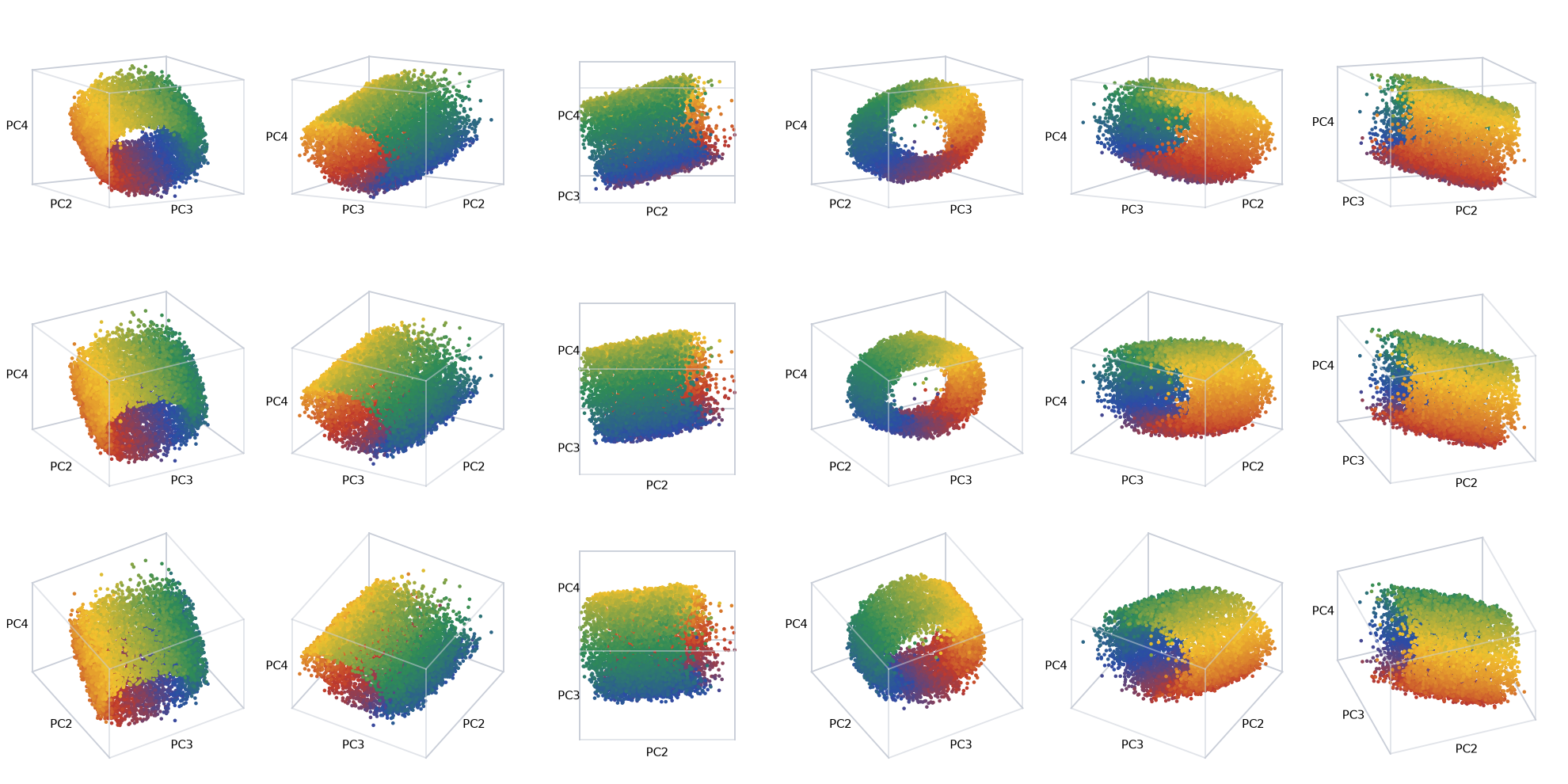}
        \caption{Cube}
        \label{fig:pca3d_cube}
    \end{subfigure}

    \caption{\textbf{PCA visualization of the learned latent representations.} Latent states are projected onto their first three principal components. TwoRoom, PushT, and Reacher exhibit approximately low-dimensional surface structure, while Cube displays a richer geometry that is not fully captured by a three-dimensional projection.}
    \label{app:fig:surfaces}
\end{figure}